\documentclass{article}

\PassOptionsToPackage{numbers, compress}{natbib}
\usepackage[final]{neurips_2026}
\usepackage[utf8]{inputenc} 
\usepackage[T1]{fontenc}    
\usepackage{hyperref}       
\hypersetup{hidelinks}
\usepackage{url}            
\usepackage{booktabs}       
\usepackage{amsfonts}       
\usepackage{nicefrac}       
\usepackage{microtype}      
\usepackage{xcolor}         

\usepackage{algorithm}
\usepackage{algorithmicx}
\usepackage{algpseudocode}
\usepackage{amsthm}
\usepackage{amsmath}

\newtheorem{theorem}{Theorem}

\newtheorem{assumption}{Assumption}
\newtheorem{proposition}{Proposition}

\newtheorem{remark}{Remark}

\usepackage{graphicx}
\usepackage[table]{xcolor} 
\usepackage[dvipsnames]{xcolor}
\usepackage{wrapfig}
\usepackage{caption}
\usepackage{placeins}

\usepackage{booktabs}
\usepackage{multirow}
\usepackage{graphicx}
\usepackage{array}
\usepackage{makecell}

\definecolor{darkgreenrank}{RGB}{0,100,0}
\definecolor{darkbluerank}{RGB}{0,70,140}
\definecolor{darkorangerank}{RGB}{180,90,0}
\definecolor{bestblue}{HTML}{E8F2FF}

\newcommand{\best}[1]{{\color{darkgreenrank}\textbf{\underline{#1}}}}
\newcommand{\second}[1]{{\color{darkbluerank}\underline{#1}}}
\newcommand{\third}[1]{{\color{darkorangerank}\underline{#1}}}
\newcommand{\bestcell}[1]{\cellcolor{bestblue}\best{#1}}
\newcommand{\ms}[2]{#1{\scriptsize$\,\pm\,#2$}}

\title{Tools as Continuous Flow for Evolving \\ Agentic Reasoning}

\author{
Tairan Huang$^{1,}$\thanks{Equal Contribution.},
Siyu Shang$^{1,*}$,
Qiang Chen$^{2,}$\thanks{Leading the Project.},
Xiu Su$^{1,}$\thanks{Corresponding Authors.},
Yi Chen$^{2,\ddagger}$
\\
$^{1}$Central South University \\
$^{2}$The Hong Kong University of Science and Technology \\
\texttt{tairanhuang@csu.edu.cn},
\texttt{siyushang@csu.edu.cn}, \\
\texttt{qchencw@connect.ust.hk},
\texttt{xiusu1994@csu.edu.cn},
\texttt{yichen@ust.hk}
}
\author{%
Tairan Huang$^{1,*}$,
Siyu Shang$^{2,*}$,
Qiang Chen$^{1,*,\dagger}$,
Xiu Su$^{2,\ddagger}$,
Yi Chen$^{1,\ddagger}$
\\
$^{1}$The Hong Kong University of Science and Technology \\
$^{2}$Central South University \\
\texttt{tairanhuang99@gmail.com},
\texttt{siyushang@csu.edu.cn}, \\
\texttt{qchencw@connect.ust.hk},
\texttt{xiusu1994@csu.edu.cn},
\texttt{yichen@ust.hk}
}

\begin{document}
\maketitle

\begingroup
\renewcommand{\thefootnote}{\fnsymbol{footnote}}
\footnotetext[1]{Equal Contribution.}
\footnotetext[2]{Leading the Project.}
\footnotetext[3]{Corresponding Authors.}
\endgroup

\begin{abstract}
Large Language Models (LLMs) have demonstrated remarkable capabilities in orchestrating external tools for complex reasoning tasks. 
However, existing methods predominantly rely on a step-wise paradigm that optimizes one discrete decision at a time, limiting their global planning perspective and causing errors to accumulate over long horizons. 
To overcome these limitations, we propose Tools as Continuous Flow for Evolving Agentic Reasoning (FlowAgent), which reconceptualizes discrete tool chaining as continuous trajectory generation within a semantic space. 
At each environment phase, a single conditional-flow solve generates the complete provisional future tool trajectory, with every latent anchor corresponding to exactly one tool at that plan position. 
We further construct a plan-level supervision and evaluation protocol from public tool-use resources by pairing each observed execution prefix with its complete remaining trajectory under strict forward-only execution. 
Theoretically, we establish formal bounds on utility convergence and prove that our continuous formulation fundamentally guarantees robust generalization and error attenuation.
Empirical evaluations show that FlowAgent achieves superior robustness and adaptability in long-horizon reasoning tasks.
Our code and datasets are available at \url{https://github.com/Tairan-Terrian/FlowAgent}.
\end{abstract}

\section{Introduction}
Large Language Models (LLMs) have evolved into capable agents that utilize external tools. 
By interacting with search engines, APIs, and transactional systems, these models ground their reasoning in real-world environments \cite{llm1,llm2,llm3,su2025large}. 
This paradigm allows agents to execute complex operations and substantially expands their utility beyond static text generation \cite{tool1,tool2,wu2024detecting}. 
Consequently, they are increasingly used to automate sophisticated workflows requiring sequential logic and dynamic environment interaction \cite{work1,work2,work3,work4}.

Recent methods formulate tool planning as an autoregressive generation task over a discrete action space \cite{ToolRL,Plan1,plan2,chen2026legato}. 
At each step, the agent observes the current context, evaluates the available options, and selects one tool for execution. 
Reinforcement learning algorithms \cite{RL1,RL2} and prompting strategies \cite{prompt,prompt2,cao2023detecting} can improve these local decisions, but the underlying process remains rooted in next-action generation and does not explicitly model the structure of the complete remaining plan.

\begin{figure*}[h]
  \centering
  \includegraphics[width=1\linewidth]{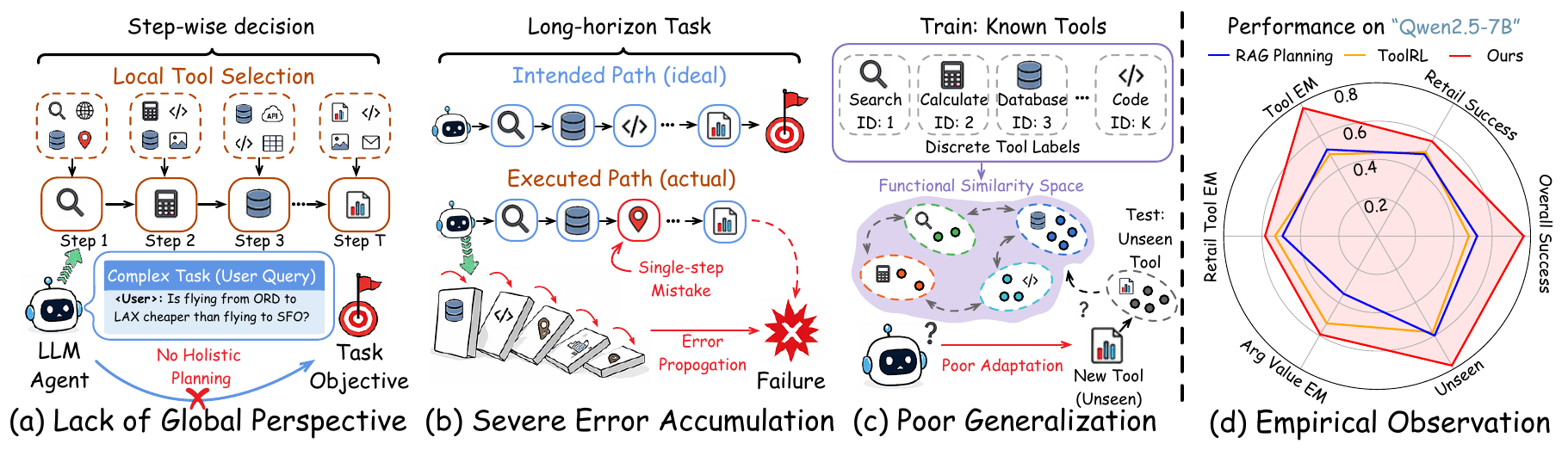}
  \caption{(a-c) Conceptual illustrations of three limitations observed in discrete tool-reasoning paradigms: \textit{limited global perspective}, \textit{error accumulation}, and \textit{weak generalization to unseen tools}.
  As evidenced in (d), the performance of the traditional step-wise paradigm is inadequate under the plan-level benchmark, highlighting the need for a paradigm shift toward global plan-level reasoning.
  }
  \label{motivation}
\end{figure*}

These step-wise paradigms exhibit three fundamental limitations when handling complex scenarios, as illustrated in Figure~\ref{motivation}. 
(1) \textit{Lack of global perspective.} Existing methods optimize localized tool selection but do not directly represent how the current decision participates in the remaining workflow. 
(2) \textit{Severe error accumulation.} In long-horizon tasks, a single suboptimal choice can alter the environment and derail the subsequent reasoning chain. 
(3) \textit{Poor generalization to unseen tools.} Treating tools as isolated categorical labels obscures their functional similarities and weakens transfer to novel schemas.

To address these inherent limitations, we propose Tools as Continuous Flow for Evolving Agentic Reasoning (FlowAgent), a novel paradigm that fundamentally shifts from myopic step-wise selection to global plan-level reasoning by reconceptualizing discrete tool chaining as continuous trajectory generation within a semantic space. 
To systematically evaluate this paradigm, we introduce the first closed-loop benchmark dedicated to plan-level agentic reasoning in dynamic real-world environments. 
Specifically, the proposed FlowAgent formulates multi-step planning as continuous trajectory evolution supervised by expert demonstrations, providing a global planning perspective to ensure coherent reasoning. 
These latent plans are subsequently translated into precise discrete actions within a closed-loop execution scheme, which dynamically absorbs environmental feedback and explicitly aligns the entire process with optimal plan-level utility. 
Theoretically, we establish formal bounds on utility convergence and mathematically prove that our continuous formulation guarantees semantic-driven tool generalization and fundamental error attenuation., encoder choices, trajectory targets, replanning horizons, and inference overhead.

The contributions of this work are summarized below.
\begin{itemize}
    \item \textbf{Plan-level Benchmark:} We introduce the plan-level closed-loop benchmark dedicated to plan-level agentic reasoning in dynamic real-world environments. 
    By integrating diverse real-world domains, this setup systematically evaluates the ability to dynamically absorb environmental feedback during multi-step execution.

    \item \textbf{Novel Paradigm:} We propose FlowAgent, a novel paradigm for evolving agentic reasoning.
    FlowAgent reconceptualizes discrete tool chaining as continuous trajectory generation within a semantic space via conditional flow matching, which shifts the decision process from myopic step-wise selection to global plan-level reasoning.
    
    \item \textbf{Theoretical Foundation:} We provide a rigorous mathematical framework that establishes explicit bounds on plan utility convergence and generalization to unseen tools. 
    Furthermore, we formally prove that our closed-loop execution scheme strictly attenuates error accumulation in long-horizon reasoning tasks.
    
    \item \textbf{Empirical Superiority:} Extensive experiments demonstrate that FlowAgent significantly outperforms baseline methods, exhibiting superior robustness in long-horizon tasks and adaptability to unseen toolsets.
\end{itemize}

\begin{figure*}[th]
    \centering
    \includegraphics[width=1\linewidth]{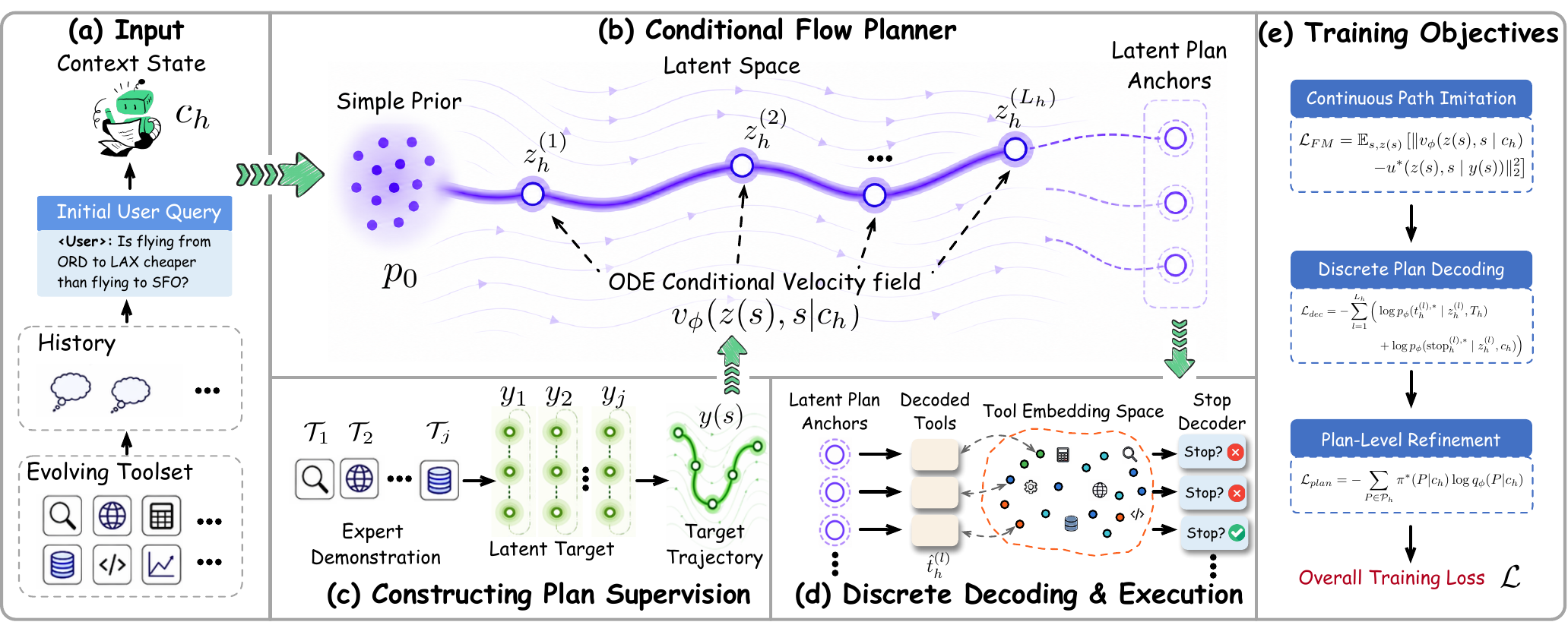}
    \caption{The overall framework of \textbf{FlowAgent}, which reconceptualizes discrete tool chaining as continuous trajectory generation to shift from myopic step-wise selection to plan-level reasoning, including \textit{Conditional Flow Planner}, \textit{Constructing Plan Supervision}, \textit{Discrete Decoding and Execution}, and \textit{Training Objectives}. 
    }
    \label{framework}
\end{figure*}

\section{Related Work}
\noindent\textbf{LLM-based Tool Reasoning.}
The efficacy of LLM agents relies heavily on their ability to interact with external tools \cite{SkillCraft,tool_agent}. 
Early paradigms attempted to mitigate this through explicit prompt engineering \cite{Engineering,Engineering2} or interleaving chain-of-thought reasoning with API invocations \cite{Chain,Chain2,wu2026injection}. 
Recent advancements have transitioned towards autonomous tool integration \cite{auto,Autotir,UniToolBench,wu2026unlearning}, formalizing the selection and execution of diverse APIs. 
ToolRL \cite{ToolRL}, ToolPlanner \cite{Toolplanner}, Agent-R1 \cite{RL2}, and AutoTIR \cite{Autotir} optimize discrete tool interactions through trajectory rewards, path feedback, or modular agentic RL. 
However, these systems predominantly operate within a discrete action space, treating tool planning as isolated categorical selections rather than globally coherent trajectories. 
In contrast, FlowAgent reconceptualizes discrete tool chaining as continuous trajectory generation within a semantic space, shifting the decision process from myopic step-wise selection to global plan-level reasoning to unlock robust adaptability to unseen toolsets.

\noindent\textbf{LLMs Agent Planning.}
LLMs Agent planning is essential for decomposing complex goals into executable sub-tasks, while simple autoregressive generation often fails to maintain long-term coherence \cite{agent_plan1,agent_plan2,agent_plan3,agent_plan4}. 
Early methods focused on linear reasoning chains \cite{linear,linear2} or iterative self-correction \cite{Correction,Correction2} to refine outputs based on internal heuristics. 
Recent research has moved toward sophisticated search-based architectures to explore and evaluate multiple reasoning branches \cite{search,search2}. 
Furthermore, specialized planners \cite{Voyager,Agentgen} incorporate environmental feedback to adjust planning trajectories during execution. 
These methods still rely on discrete search within a step-wise framework and can accumulate errors over long horizons. 
In contrast, FlowAgent bypasses the limitations of myopic discrete search by generating continuous trajectories, facilitating global plan-level reasoning that maintains structural coherence.


\section{Methodology}
\label{method}

\subsection{Problem Definition and Formalization}
We first introduce problem definition and formalization, followed by a detailed description of the design of FlowAgent.
The overall framework is shown in Figure \ref{framework}.

\noindent\textbf{Environment, State, and Tool Representation.}
Consider an agentic reasoning task evolving over discrete phases $h \in \{1, \dots, H\}$. 
At phase $h$, the system maintains an available toolset $T_h=\{t_1, \dots, t_{|T_h|}\}$, which dynamically evolves over time. 
To enable unseen tool generalization, each tool $t \in T_h$ is parameterized by a continuous semantic embedding $e_t \in \mathbb{R}^d$, jointly encoding its functional description and usage constraints. 
We define the context state at phase $h$ as $c_h = (x, \tau_{<h}, T_h)$, where $x$ denotes the initial query, and $\tau_{<h} = \big((r_1, \hat{t}_1, o_1), \dots, (r_{h-1}, \hat{t}_{h-1}, o_{h-1})\big)$ encapsulates the historical reasoning trajectory. Specifically, at each past step $i < h$, $r_i$ denotes the generated rationale, $\hat{t}_i$ is the executed tool, and $o_i$ represents the environmental observation returned after execution.

We distinguish the \emph{environment phase} $h$ from the \emph{future plan position} $l$. At one phase $h$, FlowAgent generates all $L_h$ future positions jointly. It does not invoke the planner $L_h$ times.

Instead of directly predicting discrete token-level actions, our framework formally models a continuous latent plan object:
\begin{equation}
    P_h = (z_h^{(1)}, z_h^{(2)}, \dots, z_h^{(L_h)}),
\end{equation}
where each latent anchor $z_h^{(l)} \in \mathbb{R}^d$ corresponds to a future planning step. 
Through discrete decoding, this continuous path projects to a provisional tool sequence $\hat{t}_h^{(l)} \in T_h$ for $l \in \{1, \dots, L_h\}$, with exactly one latent anchor assigned to each plan position.

\noindent\textbf{Plan-Level Objective.}
Traditional step-wise tool selection focuses on optimizing a localized policy $t_h \sim \pi(t|c_h, T_h)$.
In contrast, we elevate the objective to model a plan-level continuous distribution $P_h \sim q_\phi(P|c_h)$. 
The execution decision relies on an embedding-based decoder mapping the continuous anchor to candidate tools via a distance-based probability:
\begin{equation}
    p_\phi(t|z_h^{(l)}, T_h) \propto \exp\left(-\frac{||z_h^{(l)} - e_t||_2^2}{\epsilon}\right),
\label{eq2}
\end{equation}
where $\epsilon$ is the temperature scaling factor controlling the discrete decoding margin.

To systematically address the global structure of multi-tool tasks, the overall training target is formulated as plan utility maximization $\max_\phi \mathbb{E}[U(P_h; c_h)]$. 
The utility function $U$ explicitly unifies task performance with practical constraints:
\begin{equation}
    U(P_h; c_h) = \text{Acc}(P_h; c_h) - \lambda_{cost}\text{Cost}(P_h) - \lambda_{red}\text{Red}(P_h) + \lambda_{cons}\text{Cons}(P_h; c_h),
\label{eq3}
\end{equation}
where $\text{Acc}(\cdot)$ measures final reasoning quality, $\text{Cost}(\cdot)$ penalizes redundant calls or latency, $\text{Red}(\cdot)$ penalizes cyclic failures, $\text{Cons}(\cdot)$ measures alignment with the observation context, and $\lambda_{\{cost, red, cons\}} > 0$ are balancing hyperparameters.

\noindent\textbf{Closed-Loop Planning Objective.}
Since tool outputs inevitably alter subsequent environmental configurations, we explicitly formalize a closed-loop execution paradigm. 
Upon receiving the observation $o_h$ from the executed tool $\hat{t}_h$, the context state is deterministically updated via the operator $\mathcal{U}$:
\begin{equation}
    c_{h+1} = \mathcal{U}(c_h, \hat{t}_h, o_h).
\label{eq4}
\end{equation}
Unlike open-loop planners that rigidly execute long-horizon chains, our system dynamically resamples the latent plan conditioned on the updated state:
\begin{equation}
    P_{h+1} \sim q_\phi(P|c_{h+1}).
\end{equation}
Consequently, the approach operates as a receding-horizon latent plan generator, using real feedback to limit the propagation of stale early-stage predictions.

\subsection{Conditional Flow Planner}
We define $s \in [0,1]$ as continuous planning time within one environment phase. Given $c_h$, the latent path $z_h(s)$ is governed by an ordinary differential equation (ODE) driven by a conditional velocity field $v_\phi$:
\begin{equation}
    \frac{dz_h(s)}{ds} = v_\phi(z_h(s), s | c_h), \quad z_h(0) \sim p_0,
    \label{eq6}
\end{equation}
where $p_0$ is a simple prior distribution. Let $s_l=(l-1)/(L_h-1)$ for $L_h>1$ (and $s_1=1$ for $L_h=1$). A single numerical solve retains the path states at these positions:
\begin{equation}
    P_h = (z_h^{(1)}, z_h^{(2)}, \dots, z_h^{(L_h)}),
    \qquad z_h^{(l)}:=z_h(s_l).
\end{equation}
Each anchor $z_h^{(l)} \in \mathbb{R}^d$ serves as a continuous semantic directive for exactly one future plan position. 
Crucially, the anchors in $P_h$ are obtained together from one conditional-flow solve rather than from repeated next-tool predictions.

\subsection{Constructing Plan Supervision}
Training the conditional flow planner requires continuous target trajectories derived from discrete expert demonstrations.
Given an expert reasoning chain of length $m$, formally denoted as $\mathcal{T}^* = \big((r_1, t_1, o_1), \dots, (r_m, t_m, o_m)\big)$, we construct a dense latent plan state $y_j \in \mathbb{R}^d$ for each discrete step $j \in \{1, \dots, m\}$:
\begin{equation}
    y_j = W_t e_{t_j} + W_r \text{Enc}(r_j) + W_o \text{Enc}(o_{j-1}) + W_p \nu_j,
\end{equation}
where $o_{j-1}$ denotes the previous observation (a learned null-observation token is used for $j=1$), $e_{t_j}$ represents the semantic tool embedding, $\nu_j$ is a learnable phase embedding, and $W_{\{t,r,o,p\}}$ are projection matrices. 
By default, $\operatorname{Enc}$ is the frozen transformer trunk of Qwen2.5-7B-Instruct used only as a text feature extractor. 

For each phase $h$ of the expert trace, the observed prefix $(x,T_h,\tau_{<h})$ forms the conditioning state $c_h$, while the complete suffix $(y_h,\ldots,y_m)$ forms the trajectory target. 
Thus, a length-$m$ trace yields $m$ distinct prefix-to-remaining-trajectory pairs, and no future action is exposed in $c_h$. 
For a suffix of length $L_h=m-h+1$, relabel its anchors as $\tilde y_l=y_{h+l-1}$ for $l=1,\ldots,L_h$. 
To formulate the regression objective over normalized time $s \in [0,1]$, we construct a continuous target path $y_h(s)$ by piecewise linear interpolation at $s_l=(l-1)/(L_h-1)$:
\begin{equation}
    y_h(s) = \tilde y_l + \frac{s-s_l}{s_{l+1}-s_l}(\tilde y_{l+1}-\tilde y_l),
    \quad s\in[s_l,s_{l+1}].
\end{equation}
For $L_h=1$, the path is constant. This trajectory is the explicit regression target for the conditional velocity field.

\subsection{Discrete Decoding and Execution}
To instantiate the continuous latent plan into executable actions, the framework maps the generated latent anchor $z_h^{(l)}$ to the dynamically evolving discrete tool space $T_h$. 
According to the distance-based distribution defined above, every latent anchor is first decoded into a provisional tool:
\begin{equation}
    \hat{t}_h^{(l)} \sim p_\phi(t|z_h^{(l)}, T_h).
    \label{eq10}
\end{equation}
To infer the remaining plan length and prevent over-generation, we concurrently learn a parameterized stop mechanism. 
This module estimates the termination probability for each plan position conditioned on the local anchor and the global context state:
\begin{equation}
    p_\phi(\text{stop}|z_h^{(l)}, c_h).
\end{equation}
The first position whose probability exceeds a confidence threshold marks the provisional end of the decoded trajectory.

Crucially, we eschew the rigid open-loop execution of the entire planned sequence $P_h$. At phase $h$, the agent executes strictly the leading tool $\hat{t}_h^{(1)}$ corresponding to the first latent anchor. The system retrieves the real-world observation $o_h$ and deterministically updates the context state:
\begin{equation}
    c_{h+1} = \mathcal{U}(c_h, \hat{t}_h^{(1)}, o_h).
    \label{eq12}
\end{equation}
Following this localized execution, the remaining unexecuted anchors $(z_h^{(2)}, \dots, z_h^{(L_h)})$ are discarded. 
A new continuous latent trajectory $P_{h+1}$ is dynamically resampled conditioned on the updated state $c_{h+1}$, establishing a formally defined plan-execute-replan loop.

\subsection{Training Objectives}
To optimize the conditional flow planner and discrete execution, we propose a three-stage training pipeline, including \textit{Continuous Path Imitation}, \textit{Discrete Plan Decoding}, and \textit{Plan-Level Refinement}.

\noindent\textbf{Continuous Path Imitation.}
We first train the flow planner to imitate the expert trajectories via conditional flow matching. 
The predicted velocity field $v_\phi$ is optimized to match the teacher vector field $u^*$ induced by the interpolated target path $y(s)$:
\begin{equation}
    \mathcal{L}_{FM} = \mathbb{E}_{s, z(s)}\left[||v_\phi(z(s), s|c_h) - u^*(z(s), s|y(s))||_2^2\right].
\end{equation}

\noindent\textbf{Discrete Plan Decoding.}
We optimize the embedding-based tool decoder and the parameterized stop mechanism. 
Given the ground-truth tool identity $t_h^{(l),*}$ and the position-specific stopping indicator $\text{stop}_h^{(l),*} \in \{0,1\}$, we minimize the aggregated cross-entropy loss over the extracted anchors:
\begin{equation}
    \mathcal{L}_{dec} = -\sum_{l=1}^{L_h} \Big( \log p_\phi(t_h^{(l),*}|z_h^{(l)}, T_h) + \log p_\phi(\text{stop}_h^{(l),*}|z_h^{(l)}, c_h) \Big).
\end{equation}

\noindent\textbf{Plan-Level Refinement.}
Finally, to align the generated latent trajectories with the global multi-tool objective, we perform plan-level reinforcement learning over sampled candidate plans $P \in \mathcal{P}_h$. 
We bridge the utility ranking with policy optimization:
\begin{equation}
    \mathcal{L}_{plan} = -\sum_{P \in \mathcal{P}_h} \pi^*(P|c_h) \log q_\phi(P|c_h),
\end{equation}
where the target policy distribution $\pi^*(P|c_h)$ is weighted by holistic plan utility function $U(P; c_h)$:
\begin{equation}
    \pi^*(P|c_h) \propto q_{\text{old}}(P|c_h) \exp(U(P; c_h)),
\end{equation}
and $q_{\text{old}}(P|c_h)$ represents the generative policy from the previous optimization iteration.

The overall training target is formulated as a weighted combination of these objectives:
\begin{equation}
    \mathcal{L} = \mathcal{L}_{FM} + \lambda_1 \mathcal{L}_{dec} + \lambda_2 \mathcal{L}_{plan} + \lambda_3 \mathcal{L}_{cons},
\end{equation}
where $\lambda_{\{1,2,3\}} > 0$ are balancing coefficients, and $\mathcal{L}_{cons} = \frac{1}{L_h} \sum_{l=1}^{L_h} || z_h^{(l)} - W_{c} \text{Enc}(c_h) ||_2^2$ acts as an auxiliary constraint explicitly regularizing the semantic consistency between the generated latent anchors and the prevailing observational context.

\section{Theoretical Analysis}
\label{Theorem}
We next connect continuous path approximation to discrete decoding, utility deviation, unseen-tool transfer, and feedback-conditioned error propagation. 
These are finite-horizon sufficient-condition results, not universal guarantees: each conclusion applies only when the corresponding geometric, margin, regularity, and feedback assumptions below hold.

\subsection{Foundational Assumptions}
To ensure the tractability of the latent planning problem, we introduce the following structural assumptions based on the system's geometric properties.

\begin{assumption}[Semantic Regularity]\label{asm:smoothness}
The train and test tool embeddings lie in a compact region of a smooth manifold $\mathcal{M}\subset\mathbb{R}^d$. On this region, the excess planning risk induced by replacing one tool embedding with another is $C_{\mathcal M}$-Lipschitz in their Euclidean distance.
\end{assumption}

\begin{assumption}[Decoding Margin]\label{asm:margin}
For every expert anchor $y_j$ with ground-truth tool $t^*$, there exists $\Delta_{\mathrm{dec}}>0$ such that $\|y_j-e_{t^*}\|_2+\Delta_{\mathrm{dec}}\le\|y_j-e_t\|_2$ for all $t\in T_h\setminus\{t^*\}$.
\end{assumption}

\begin{assumption}[Utility Regularity]\label{asm:utility}
For the finite horizons considered, $U$ is bounded and Lipschitz in both the latent plan and encoded context: $|U(P;c)-U(P';c')|\le L_P\sum_l\|z^{(l)}-z'^{(l)}\|_2+L_c d_c(c,c')$.
\end{assumption}

\begin{assumption}[Feedback Stability]\label{asm:stability}
Let $e_h=d_c(c_h,c_h^*)$ be the encoded-context deviation from an oracle execution and let $a_h\ge0$ collect the new flow, decoding, and observation error introduced at phase $h$. The feedback-conditioned update satisfies $e_{h+1}\le\rho e_h+a_h$ for some local $0\le\rho<1$ on the evaluated trajectory region.
\end{assumption}

\subsection{Propositions and Theorems}

\begin{proposition}[Endpoint Error Bound]\label{prop:endpoint}
Let the teacher field be $K_u$-Lipschitz and let the pointwise velocity error be at most $\delta$. Then $\eta_\phi:=\mathbb E\|z(1)-y(1)\|_2\le C_{\mathrm{flow}}\delta$, where $C_{\mathrm{flow}}=(e^{K_u}-1)/K_u$ (or $1$ when $K_u=0$). For nearest-neighbor/MAP decoding,
\begin{equation}
\Pr(\hat t_{\mathrm{MAP}}\ne t^*)\le
\min\!\left\{1,\frac{2\eta_\phi}{\Delta_{\mathrm{dec}}}\right\}.
\label{eq:endpoint-main}
\end{equation}
\end{proposition}
\begin{remark}
If the endpoint error is deterministically below $\Delta_{\mathrm{dec}}/2$, the correct tool remains the unique nearest neighbor; Eq.~\eqref{eq:endpoint-main} covers the expected-error case.
The detailed proof is provided in Appendix \ref{appendix:proof1}.
\end{remark}

\begin{theorem}[Plan Utility Convergence]\label{thm:convergence}
Let $P_h^*$ be the expert plan and let $P_h$ be the generated plan of length $L_h$. Suppose the anchor neighborhood has positive squared-distance gaps $g_l>0$, the maximum one-tool utility change is $B_U$, and $N_h=|T_h|$. If $\eta_h$ bounds the expected integrated velocity error, then
\begin{equation}
\begin{split}
|U(P_h^*;c_h)-\mathbb E U(P_h;\hat c_h)|
\le {}& L_P L_h C_{\mathrm{flow}}\eta_h \\
&+B_U(N_h-1)\sum_{l=1}^{L_h}e^{-g_l/\epsilon}
+L_c\beta_o\mathbb E\|o-\hat o\|_2 .
\end{split}
\label{eq18}
\end{equation}
Consequently, the expected utility approaches the oracle value as the flow and observation errors vanish and the decoder becomes deterministic while retaining positive gaps.
\end{theorem}
\begin{remark}
The three terms isolate continuous path approximation, finite-temperature decoding, and feedback perturbation; the theorem does not assert convergence when a margin collapses or the stability assumptions fail.
The detailed proof is provided in Appendix \ref{appendix:theorem1}.
\end{remark}

\begin{theorem}[Unseen Tool Generalization]\label{thm:generalization}
For every unseen tool, let $\pi(t)$ be its nearest training proxy. Under Assumption~\ref{asm:smoothness},
\begin{equation}
    \mathbb{E}[\operatorname{Err}_{\mathrm{unseen}}]
    \le \mathbb{E}[\operatorname{Err}_{\mathrm{proxy}}]
    + C_{\mathcal M}(\delta_{\mathrm{shift}}+\epsilon_{\mathrm{cover}}),
    \label{eq19}
\end{equation}
where $\delta_{\mathrm{shift}}$ is the distance from the unseen region to the training manifold region and $\epsilon_{\mathrm{cover}}$ is the covering radius of the training embeddings.
\end{theorem}
\begin{remark}
The bound separates the proxy's existing planning error from the additional penalty caused by semantic shift and incomplete training coverage.
The detailed proof is in Appendix \ref{appendix:theorem2}.
\end{remark}

\begin{proposition}[Closed-loop Error Contraction]\label{prop:closed_loop}
Under Assumption~\ref{asm:stability}, the cumulative closed-loop context error obeys
\begin{equation}
E_{\mathrm{closed}}\le
\sum_{i=1}^{H-1}a_i\frac{1-\rho^{H-i}}{1-\rho}
\le
\sum_{i=1}^{H-1}(H-i)a_i=:E_{\mathrm{open}}.
\label{eq:closed-main}
\end{equation}
The second inequality is strict whenever some $a_i>0$ has at least two subsequent phases and $\rho<1$.
\end{proposition}
\begin{remark}
Here $E_{\mathrm{open}}$ is the no-correction accumulation envelope for the same local errors. The result characterizes replanning only inside the assumed contractive region; it does not cover non-contractive or adversarial updates.
The detailed proof is provided in Appendix \ref{appendix:proof2}.
\end{remark}

\section{Experiments}
\subsection{Benchmark and Experimental Setup}
\label{sec:setup}

\noindent\textbf{Plan-level Supervision and Evaluation Protocol.}
We build on four public sources: \textit{$\tau^2$-bench} \cite{tau}, \textit{ToolBench} \cite{tool_bench}, \textit{API-Bank} \cite{api_bank}, and \textit{GTA} \cite{gta}. 
Our contribution is not a new closed-loop environment: $\tau^2$-bench supplies the executable environment. 
Instead, we standardize schemas and traces and create a plan-level protocol in which every observed expert prefix is paired with its complete remaining trajectory. 
The unified data contain 6,865 tasks, 3,930 tool schemas, and 599 test-only tools. 
For stateful executable evaluation, each prediction is replayed against real tools, policies, and databases under strict forward-only execution; the submitted split contains 140 trajectories and the expanded study uses up to 800.
More detailed descriptions of the dataset construction and execution configurations are provided in Appendix \ref{sec:appendix_benchmark}.

\newcommand{\rotmodel}[1]{%
  \rotatebox[origin=c]{90}{\makecell[c]{#1}}%
}
\begin{table*}[t]
\centering
\caption{Main results on our proposed plan-level benchmark. All models are evaluated under a strict forward-execution regime. Highlighted are the results ranked \best{first}, \second{second}, and \third{third}.}
\label{tab:main_results}
\setlength{\tabcolsep}{8pt}
\renewcommand{\arraystretch}{1.1}

\resizebox{\textwidth}{!}{
\begin{tabular}{clcccccc}
\toprule
\multirow{3}{*}{\textbf{Executor}} & \multirow{3}{*}{\textbf{Methods}} &
\multicolumn{2}{c}{\textbf{Execution Success}} & 
\multicolumn{2}{c}{\textbf{Planning Accuracy}} &
\multicolumn{2}{c}{\textbf{Execution \& Grounding}} \\
\cmidrule(lr){3-4} \cmidrule(lr){5-6} \cmidrule(lr){7-8}
& & \textbf{Overall} & \textbf{Retail} & \textbf{Tool} & \textbf{Retail} & \textbf{Pred} & \textbf{Arg Value} \\
& & \textbf{Success $\uparrow$} & \textbf{Success $\uparrow$} & \textbf{EM $\uparrow$} & \textbf{Tool EM $\uparrow$} & \textbf{Exec $\uparrow$} & \textbf{EM $\uparrow$} \\
\midrule

\multirow{5}{*}{\rotmodel{\textbf{Qwen2.5-7B}\\\textbf{-Instruct}}}
& Direct-LLM               & \ms{0.1214}{0.0186} & \ms{0.2208}{0.0215} & \ms{0.1286}{0.0179} & \ms{0.2338}{0.0231} & \ms{0.9412}{0.0068} & \ms{0.4308}{0.0197} \\
& Standard SFT             & \second{\ms{0.6571}{0.0126}} & \ms{0.3766}{0.0174} & \second{\ms{0.6714}{0.0118}} & \ms{0.4026}{0.0192} & \ms{0.9268}{0.0075} & \third{\ms{0.4387}{0.0156}} \\
& RAG Planning             & \third{\ms{0.5214}{0.0158}} & \third{\ms{0.4935}{0.0189}} & \third{\ms{0.5214}{0.0163}} & \third{\ms{0.4935}{0.0195}} & \best{\ms{1.0000}{0.0000}} & \ms{0.3478}{0.0172} \\
& ToolRL    & \ms{0.4786}{0.0289} & \second{\ms{0.5065}{0.0316}} & \ms{0.4929}{0.0274} & \second{\ms{0.5325}{0.0331}} & \second{\ms{0.9643}{0.0098}} & \second{\ms{0.5257}{0.0295}} \\
& \textbf{FlowAgent (Ours)} & \best{\ms{0.7643}{0.0098}} & \best{\ms{0.5714}{0.0137}} & \best{\ms{0.7745}{0.0091}} & \best{\ms{0.5892}{0.0152}} & \third{\ms{0.9619}{0.0061}} & \best{\ms{0.5929}{0.0124}} \\
\midrule

\multirow{5}{*}{\rotmodel{\textbf{Llama-3.2-3B}\\\textbf{-Instruct}}}
& Direct-LLM               & \ms{0.2000}{0.0207} & \ms{0.0779}{0.0169} & \ms{0.2071}{0.0198} & \ms{0.0909}{0.0182} & \ms{0.8060}{0.0116} & \ms{0.1225}{0.0175} \\
& Standard SFT             & \second{\ms{0.6214}{0.0152}} & \third{\ms{0.3117}{0.0201}} & \second{\ms{0.6286}{0.0144}} & \third{\ms{0.3247}{0.0213}} & \third{\ms{0.9318}{0.0084}} & \second{\ms{0.4348}{0.0187}} \\
& RAG Planning             & \third{\ms{0.5286}{0.0176}} & \second{\ms{0.5065}{0.0198}} & \third{\ms{0.5500}{0.0169}} & \second{\ms{0.5455}{0.0211}} & \ms{0.9091}{0.0102} & \third{\ms{0.3241}{0.0181}} \\
& ToolRL     & \ms{0.1357}{0.0361} & \ms{0.1169}{0.0328} & \ms{0.1500}{0.0348} & \ms{0.1429}{0.0355} & \second{\ms{0.9362}{0.0124}} & \ms{0.1028}{0.0307} \\
& \textbf{FlowAgent (Ours)} & \best{\ms{0.7571}{0.0112}} & \best{\ms{0.5584}{0.0159}} & \best{\ms{0.7714}{0.0103}} & \best{\ms{0.5844}{0.0167}} & \best{\ms{0.9524}{0.0072}} & \best{\ms{0.5573}{0.0136}} \\
\bottomrule
\end{tabular}
}
\end{table*}

\begin{table*}[h]
\centering
\caption{Ablation studies of proposed components on the Qwen2.5-7B executor. Additional ablation results can be found in Appendix \ref{sec:additional_exp}.}
\label{tab:ablation}
\setlength{\tabcolsep}{6pt}
\renewcommand{\arraystretch}{1.1}
\resizebox{\textwidth}{!}{
\begin{tabular}{ccc cccccc}
\toprule
\textbf{Flow Planner} & \textbf{Con. State} & \textbf{Stop Mech.} & \textbf{Overall Suc.} & \textbf{Retail Suc.} & \textbf{Tool EM} & \textbf{Retail Tool EM} & \textbf{Pred Exec} & \textbf{Arg Value EM} \\
\midrule
$\times$ & $\times$ & $\times$ & 0.6571 & 0.3766 & 0.6714 & 0.4026 & 0.9268 & 0.4387 \\
\checkmark & $\times$ & $\times$ & 0.6786 & 0.4156 & 0.6857 & 0.4286 & 0.9381 & 0.5099 \\
\checkmark & \checkmark & $\times$ & 0.7357 & 0.5195 & 0.7429 & 0.5325 & \cellcolor{bestblue}\textbf{\best{0.9619}} & 0.5613 \\
\checkmark & \checkmark & \checkmark & 
\cellcolor{bestblue}\textbf{\best{0.7643}} & 
\cellcolor{bestblue}\textbf{\best{0.5714}} & 
\cellcolor{bestblue}\textbf{\best{0.7745}} & 
\cellcolor{bestblue}\textbf{\best{0.5892}} & 
\cellcolor{bestblue}\textbf{\best{0.9619}} & 
\cellcolor{bestblue}\textbf{\best{0.5929}} \\
\bottomrule
\end{tabular}
}
\end{table*}

\noindent\textbf{Baselines.}
We compare FlowAgent with Direct-LLM, Standard SFT, RAG Planning, and ToolRL \cite{ToolRL} in the main table. 
Expanded comparisons add ToolPlanner \cite{Toolplanner}, Agent-R1 \cite{RL2}, and AutoTIR \cite{Autotir} under matched executors and budgets; implementation details and full results appear in Appendix \ref{sec:appendix_baselines} and Appendix \ref{sec:additional_exp}.

\noindent\textbf{Metrics.}
To holistically evaluate the global planning capabilities and execution robustness, we employ a multi-dimensional metric suite, including Overall Success, Retail Success, Tool EM, Retail Tool EM, Pred Exec, and Arg Value EM. 
The detailed formulations for all evaluation metrics are provided in Appendix \ref{sec:appendix_metrics}.

\subsection{Experimental Results}
\noindent\textbf{Main Results.} 
As shown in Table \ref{tab:main_results}, FlowAgent consistently outperforms all baselines across both executors. 
Compared to the Standard SFT baseline lacking the continuous flow matching prior, our method improves the overall success rate on Qwen2.5-7B from 0.6571 to 0.7643. 
This performance gap indicates that our generative prior offers better state conditioning than simple retrieval hints or discrete policy optimization. 
Finally, the consistent performance gains observed during the Llama-3.2-3B evaluation validate FlowAgent as a generalized planning framework capable of bridging abstract reasoning and structured action execution in long-horizon reasoning tasks.
\begin{wrapfigure}{r}{0.29\textwidth}
  \vspace{4pt}
  \centering
  \captionsetup{skip=2.5pt}
  \includegraphics[width=1\linewidth]{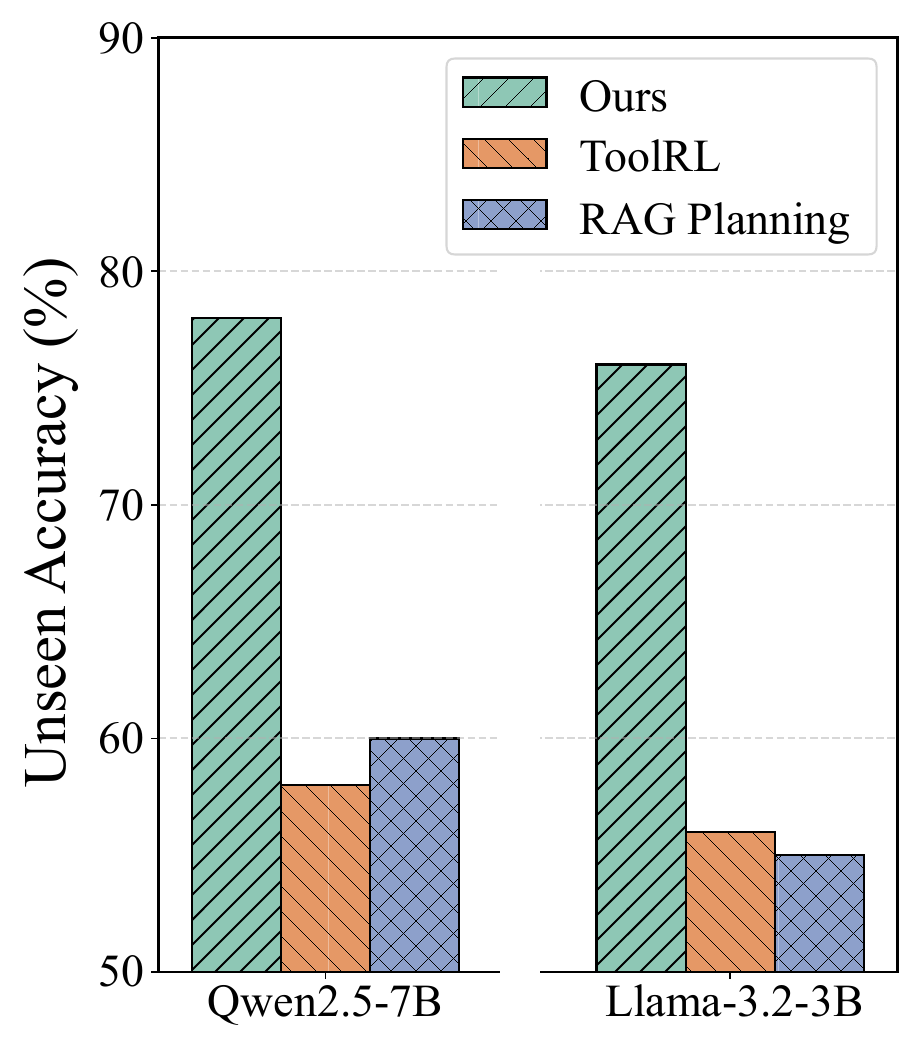}
  \caption{Generalization performance on unseen tools. 
  }
  \label{fig:unseen_generalization}
\end{wrapfigure}

\noindent\textbf{Ablation Studies.}
To evaluate the contribution of each component, we conduct the ablation study shown in Table~\ref{tab:ablation}. 
Starting from Standard SFT, introducing the Conditional Flow Planner increases Overall Success from 0.6571 to 0.6786 and improves Arg Value EM from 0.4387 to 0.5099. 
Adding the observation-grounded Context State produces the largest improvement and raises Overall Success to 0.7357. 
Finally, the Parameterized Stop Mechanism further improves Retail Success from 0.5195 to 0.5714 and Retail Tool EM from 0.5325 to 0.5892. 
These results demonstrate that each component provide complementary benefits for long-horizon tool orchestration.

\noindent\textbf{Generalization Results.} 
\label{sec:generalization}
To evaluate generalization to novel tool schemas, we measure planning accuracy on a strictly held-out set of unseen tools, as shown in Figure~\ref{fig:unseen_generalization}. 
FlowAgent achieves a planning accuracy of 78\% on the Qwen2.5-7B executor and 76\% on the Llama-3.2-3B executor. 
FlowAgent's consistent performance indicates that continuous semantic trajectories capture tool relationships that remain useful when the available toolset changes.

\noindent\textbf{RL Baseline Results.}
To determine whether the improvement extends beyond ToolRL, we compare FlowAgent with ToolPlanner, Agent-R1, and AutoTIR, as shown in Figure~\ref{fig:rl_breadth}. 
FlowAgent nevertheless improves Overall Success over Agent-R1 from 0.7186 to 0.7643 on Qwen2.5-7B and from 0.6929 to 0.7571 on Llama-3.2-3B. 
These results confirm that the improvement is consistent across recent RL planner families and is not explained by executability alone.

\noindent\textbf{Larger-Executor Results.}
To evaluate whether FlowAgent remains effective as executor capacity increases, we repeat the comparison on four larger instruction-tuned models, as shown in Figure~\ref{fig:scale_large}. 
ToolRL becomes stronger on larger executors, particularly on the retail subset, but remains below FlowAgent in task success. 
The consistent improvements within both model families demonstrate that complete-trajectory supervision remains beneficial as the underlying executor scales.

\begin{minipage}[t]{0.64\textwidth}
  \vspace{0pt}\centering
  \begin{minipage}[t]{0.49\linewidth}
    \centering\includegraphics[width=\linewidth]{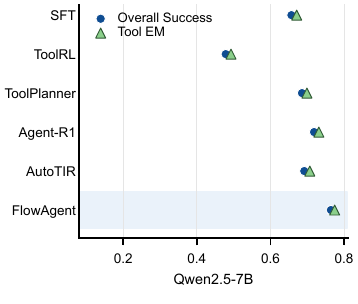}
  \end{minipage}\hfill
  \begin{minipage}[t]{0.49\linewidth}
    \centering\includegraphics[width=\linewidth]{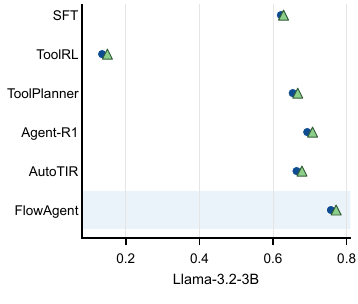}
  \end{minipage}
  \captionsetup{font=small,skip=4pt,hypcap=false}
  \captionof{figure}{Comparison with recent RL tool-use methods.}
  \label{fig:rl_breadth}
\end{minipage}\hfill
\begin{minipage}[t]{0.33\textwidth}
  \vspace{0pt}\centering
  \includegraphics[width=\linewidth]{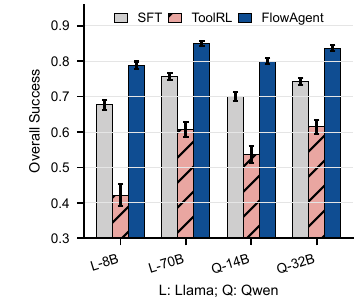}
  \captionsetup{font=small,skip=4pt,hypcap=false}
  \captionof{figure}{Larger executors.}
  \label{fig:scale_large}
\end{minipage}

\begin{minipage}[t]{0.315\textwidth}
  \vspace{0pt}\centering
  \includegraphics[width=\linewidth]{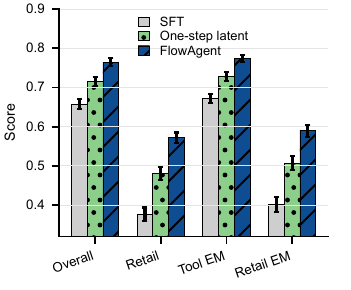}
  \captionsetup{font=small,skip=4pt,hypcap=false}
  \captionof{figure}{Prediction target.}
  \label{fig:prediction_target}
\end{minipage}\hfill
\begin{minipage}[t]{0.315\textwidth}
  \vspace{0pt}\centering
  \includegraphics[width=\linewidth]{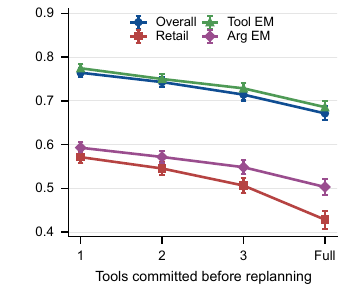}
  \captionsetup{font=small,skip=4pt,hypcap=false}
  \captionof{figure}{Replanning frequency.}
  \label{fig:replanning_horizon}
\end{minipage}\hfill
\begin{minipage}[t]{0.315\textwidth}
  \vspace{0pt}\centering
  \includegraphics[width=\linewidth]{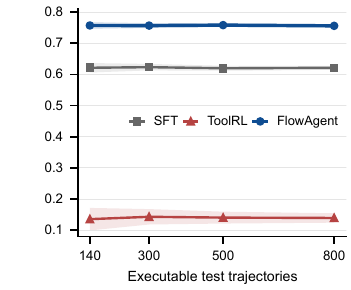}
  \captionsetup{font=small,skip=4pt,hypcap=false}
  \captionof{figure}{Test-set sizes.}
  \label{fig:test_stability}
\end{minipage}

\subsection{Additional Analysis and Discussion.}
\noindent\textbf{Prediction-Target Analysis.}
To distinguish complete-trajectory generation from a stronger one-step semantic representation, we compare three prediction targets in Figure~\ref{fig:prediction_target}. 
One-step latent regression improves Overall Success from 0.6571 to 0.7149 and Retail Success from 0.3766 to 0.4805 over discrete SFT, confirming that semantic regression itself is useful. 
FlowAgent further improves these metrics to 0.7643 and 0.5714 when the entire future latent trajectory is supervised.

\noindent\textbf{Replanning Analysis.}
Figure~\ref{fig:replanning_horizon} evaluates how many predicted tools should be executed before incorporating new feedback. 
Retail Success exhibits a larger decline from 0.5714 to 0.4286 because database-changing observations can invalidate later actions in the provisional trajectory. 
This result demonstrates that complete-trajectory generation supplies global planning structure, while frequent feedback-conditioned replanning remains necessary for reliable execution.

\noindent\textbf{Test-Scale Analysis.}
To examine whether the submitted 140-trajectory test set produces a favorable estimate, we construct independent test sets containing 300, 500, and 800 executable trajectories. 
As shown in Figure~\ref{fig:test_stability}, FlowAgent's Overall Success remains between 0.7559 and 0.7582 across all three test sizes, and Tool EM varies by at most 0.0033. 
These results show that the original performance advantage is stable under substantially expanded executable evaluation.

\noindent\textbf{Efficiency Analysis.}
To evaluate computational efficiency, we compare the training overhead of FlowAgent with the reinforcement-learning baseline ToolRL. 
As reported in Table~\ref{tab:efficiency}, FlowAgent completes optimization in 12.29 minutes, whereas ToolRL requires 48.75 minutes. 
The peak memory on the primary GPU is reduced from 78.69 to 52.63 GiB, and the average peak memory across all three GPUs decreases from 68.36 to 55.14 GiB. 
\begin{wraptable}{r}{0.8\textwidth}  
\vspace{-4pt}
\centering
\caption{Computational efficiency comparison. We report the total training time and the peak memory consumption across three distinct GPUs.}
\label{tab:efficiency}
\resizebox{1\linewidth}{!}{
\begin{tabular}{l ccccc}
\toprule
\textbf{Method} & \textbf{Training Time $\downarrow$} & \textbf{GPU 1 Peak $\downarrow$} & \textbf{GPU 2 Peak $\downarrow$} & \textbf{GPU 3 Peak $\downarrow$} & \textbf{Avg Peak $\downarrow$} \\
\midrule
ToolRL & 48.75 min & 78.69G & 63.22G & 63.16G & 68.36G \\
\textbf{FlowAgent} & 
\cellcolor{bestblue}\textbf{\best{12.29 min}}  & \cellcolor{bestblue}\textbf{\best{52.63G}}  & \cellcolor{bestblue}\textbf{\best{59.91G}} & \cellcolor{bestblue}\textbf{\best{52.88G}} & \cellcolor{bestblue}\textbf{\best{55.14G}} \\
\bottomrule
\end{tabular}
}
\end{wraptable}
This improvement results from learning the conditional velocity field through direct trajectory regression, which avoids the repeated environment rollouts and policy-update overhead required by discrete RL optimization. 
Therefore, FlowAgent improves planning performance while substantially reducing training time and memory consumption.

\noindent
\begin{minipage}[t]{0.64\textwidth}
  \vspace{0pt}\centering
  \begin{minipage}[t]{0.49\linewidth}
    \centering\includegraphics[width=\linewidth]{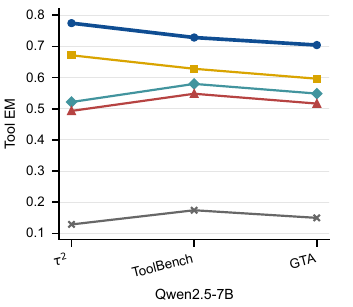}
  \end{minipage}\hfill
  \begin{minipage}[t]{0.49\linewidth}
    \centering\includegraphics[width=\linewidth]{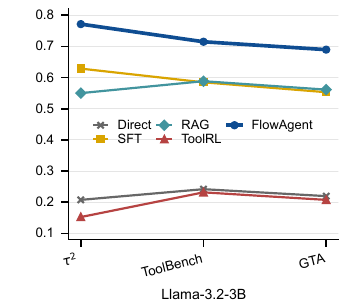}
  \end{minipage}
  \captionsetup{font=small,skip=4pt,hypcap=false}
  \captionof{figure}{Tool EM across task sources.}
  \label{fig:cross_source}
\end{minipage}\hfill
\begin{minipage}[t]{0.33\textwidth}
  \vspace{0pt}\centering
  \includegraphics[width=\linewidth]{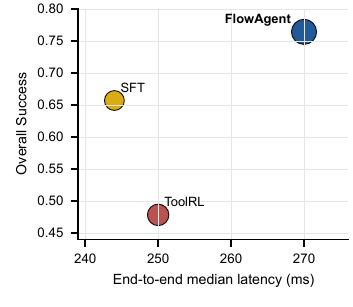}
  \captionsetup{font=small,skip=4pt,hypcap=false}
  \captionof{figure}{Inference latency versus.}
  \label{fig:deployment_cost}
\end{minipage}

\begin{minipage}[t]{0.475\textwidth}
\vspace{0pt}\centering
\captionsetup{font=small,skip=3pt,hypcap=false}
\captionof{table}{Temperature sensitivity.}
\label{Parameter}
\resizebox{\linewidth}{!}{
\begin{tabular}{cccc}
\toprule
\textbf{Temperature} & \textbf{Overall Suc. $\uparrow$} & \textbf{Retail Suc. $\uparrow$} & \textbf{Tool EM $\uparrow$} \\
\midrule
$\epsilon=0.01$ & 0.6957 & 0.5165 & 0.7029 \\
$\epsilon=0.05$ & 0.7351 & 0.5452 & 0.7504 \\
$\epsilon=0.10$ & \cellcolor{bestblue}\textbf{\best{0.7643}} & \cellcolor{bestblue}\textbf{\best{0.5714}} & \cellcolor{bestblue}\textbf{\best{0.7745}} \\
$\epsilon=0.20$ & 0.7214 & 0.5329 & 0.7286 \\
$\epsilon=0.50$ & 0.6729 & 0.4975 & 0.6971 \\
$\epsilon=0.60$ & 0.6540 & 0.4826 & 0.6815 \\
\bottomrule
\end{tabular}}
\end{minipage}\hfill
\begin{minipage}[t]{0.505\textwidth}
\vspace{0pt}\centering
\captionsetup{font=small,skip=3pt,hypcap=false}
\captionof{table}{Retail decision phases.}
\label{tab:retail_analysis}
\resizebox{\linewidth}{!}{
\begin{tabular}{llcc}
\toprule
\textbf{Method} & \textbf{Decision Case} & \textbf{Success $\uparrow$} & \textbf{Arg Value EM $\uparrow$} \\
\midrule
\multirow{3}{*}{FlowAgent}
& Initial Lookup & 0.7000 & 0.6833 \\
& Detail Lookup & 0.7692 & 0.6923 \\
& Database Operation & \cellcolor{bestblue}\textbf{\best{0.3529}} & \cellcolor{bestblue}\textbf{\best{0.4202}} \\
\midrule
\multirow{3}{*}{- w/o Stop Mech.}
& Initial Lookup & 0.7000 & 0.6833 \\
& Detail Lookup & 0.7692 & 0.6923 \\
& Database Operation & 0.1176 & 0.2878 \\
\midrule
Oracle Selection & Database Operation & 0.7059 & 0.6190 \\
\bottomrule
\end{tabular}}
\end{minipage}

\begin{wrapfigure}{r}{0.30\textwidth}
  \vspace{-10pt}
  \centering
  \captionsetup{font=small,skip=2.5pt}
  \includegraphics[width=\linewidth]{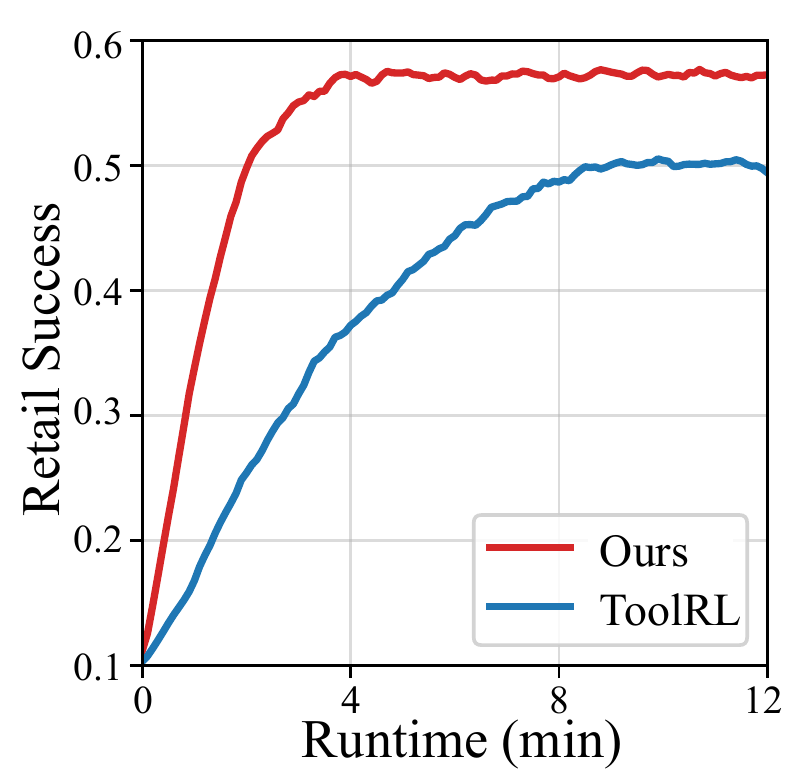}
  \caption{Convergence speed.}
  \label{convergence}
\end{wrapfigure}

\noindent\textbf{Cross-Source Analysis.}
To test whether the improvement is confined to a single benchmark source, we regroup the generated predictions by $\tau^2$-bench, ToolBench, and GTA. 
Figure~\ref{fig:cross_source} shows that FlowAgent obtains the highest Tool EM on every source for both executors. 
The results demonstrate that FlowAgent's advantage is not produced by one benchmark's tool distribution.

\noindent\textbf{Inference-Cost Analysis.}
Figure~\ref{fig:deployment_cost} reports the deployment trade-off introduced by numerical flow integration. 
FlowAgent uses 16 velocity-field function evaluations per generated plan, increasing median planning latency from 171 to 194 ms and median end-to-end latency from 244 to 270 ms relative to SFT. 


\noindent\textbf{Parameter Analysis.}
To evaluate robustness to the decoder configuration, we vary the temperature $\epsilon$ and report the results in Table~\ref{Parameter}. 
FlowAgent obtains its best performance at $\epsilon=0.10$, reaching 0.7643 Overall Success, 0.5714 Retail Success, and 0.7745 Tool EM. 
The smooth performance variation around the optimum indicates that FlowAgent is not sensitive to a narrowly tuned temperature value.

\noindent\textbf{Retail Decision Analysis.}
The main evaluation shows that state-changing database operations are the most difficult decisions in the retail subset. 
We therefore separate the initial lookup, detail lookup, and database-operation phases in Table~\ref{tab:retail_analysis}. 
The oracle tool-selection result reaches 0.7059 success, showing that state-conditioned operation selection and argument grounding remain the primary bottlenecks in complex retail tasks.

\noindent\textbf{Convergence Analysis.}
To compare optimization behavior, we track Retail Success as a function of training time in Figure~\ref{convergence}. 
FlowAgent improves rapidly during the first few minutes and stabilizes after approximately 4 minutes at a higher Retail Success score. 
In contrast, ToolRL requires about 9 minutes to approach convergence. 
This result is consistent with the training-cost comparison and shows that FlowAgent improves both convergence speed and final task performance.

\section{Conclusion}
In this paper, we present \textbf{FlowAgent}, a novel framework that reconceptualizes tools as continuous flow for evolving agentic reasoning. 
FlowAgent leverages conditional flow matching to generate continuous trajectories and provides a global perspective to ensure coherent reasoning and robust tool execution.
Theoretically, we establish formal bounds on utility convergence and prove that our formulation fundamentally guarantees error attenuation.
Extensive empirical evaluations demonstrate that FlowAgent achieves superior adaptability in long horizon reasoning tasks.

\clearpage

\begin{ack}
This research is partially supported by grants from the National Natural Science Foundation of China (\#72495125 and \#W2511077).
Yi Chen is supported by the Hong Kong Research Grants Council, Early Career Scheme Fund [Grant 26508924] and Hong Kong RGC Theme-based Research Scheme T32-615-24/R.
This research is also funded by the National Natural Science Foundation of China 62406347.
\end{ack}

\bibliographystyle{unsrtnat}
\bibliography{mybibfile}

\newpage
\appendix
\textbf{Appendix Contents:}
\begin{itemize}
    \item \ref{appendix:proof} Proofs 
    \item \ref{sec:exp_details} Experimental Details
    \item \ref{sec:additional_exp} Additional Experiments
    \item \ref{sec::impact} Broader Impact

\end{itemize}





\section{Proofs}
\label{appendix:proof}

This appendix provides the complete proofs for the theoretical results stated in Section \ref{Theorem}.
All notation follows the main text. At phase $h$, the context is $c_h=(x,\tau_{<h},T_h)$,
the generated plan is $P_h=(z_h^{(1)},\ldots,z_h^{(L_h)})$, and the expert plan is
$P_h^*=(y_h(s_1),\ldots,y_h(s_{L_h}))$. We reserve $U(P;c)$ for plan utility and
$\mathcal U(c,t,o)$ for the deterministic context update.

\paragraph{Auxiliary notation.}
For two phase-$h$ latent plans
$P=(z^{(1)},\ldots,z^{(L_h)})$ and
$P'=(z'^{(1)},\ldots,z'^{(L_h)})$, we use the plan metric:
\begin{equation}
\label{eq:app-plan-metric}
  \|P-P'\|_2 := \sum_{l=1}^{L_h}\|z^{(l)}-z'^{(l)}\|_2,
\end{equation}
which is the plan metric used in Assumption~\ref{asm:utility}. For a finite phase-$h$ toolset,
let $N_h:=|T_h|$. All constants below are local to the finite horizon and to the compact semantic
region visited by the train and test embeddings.

\subsection{Proof of Proposition~1}
\label{appendix:proof1}

\paragraph{Proposition~1 (Endpoint Error Bound).}
Let $z(1)$ be the flow endpoint and let $y(1)$ be the expert target anchor associated with the
correct tool $t^*$. Let $z^*(s)$ denote the exact teacher trajectory induced by
$u^*(\cdot,s\mid y(s))$ under the same initial prior sample as Eq. \eqref{eq6}, so that
$z^*(1)=y(1)$ by construction of the conditional flow target. Suppose that, on the compact region visited by the numerical solver, $u^*(\cdot,s\mid y(s))$ is $K_u$-Lipschitz in its first argument and
\begin{equation}
\label{eq:app-uniform-velocity-error}
  \|v_\phi(z(s),s\mid c_h)-u^*(z(s),s\mid y(s))\|_2\le \delta
  \qquad \forall s\in[0,1].
\end{equation}
Then the flow integration error
\begin{equation}
\label{eq:app-flow-integration-error}
  \eta_\phi:=\mathbb E\|z(1)-y(1)\|_2
\end{equation}
satisfies $\eta_\phi\le C_{\mathrm{flow}}\delta$, where
$C_{\mathrm{flow}}=(e^{K_u}-1)/K_u$ if $K_u>0$ and $C_{\mathrm{flow}}=1$ if $K_u=0$.
Consequently, for the MAP decoder induced by Eq. \eqref{eq2}:
\begin{equation}
\label{eq:app-prop1-bound}
  \mathbb P\{\hat t_{\mathrm{MAP}}\neq t^*\}
  \le \min\left\{1,\frac{2\eta_\phi}{\Delta_{\mathrm{dec}}}\right\}
  \le \min\left\{1,\frac{2C_{\mathrm{flow}}\delta}{\Delta_{\mathrm{dec}}}\right\}.
\end{equation}
Thus the decoding error is controlled by the ratio between the flow integration error and the
decoding margin.

\noindent\textbf{Proof.}
Let $e(s):=\|z(s)-z^*(s)\|_2$. Since $z(0)$ and $z^*(0)$ use the same prior sample,
$e(0)=0$. The ODEs for $z(s)$ and $z^*(s)$ give
\begin{align}
  \frac{d}{ds}e(s)
  &\le \|v_\phi(z(s),s\mid c_h)-u^*(z^*(s),s\mid y(s))\|_2 \\
  &\le \|v_\phi(z(s),s\mid c_h)-u^*(z(s),s\mid y(s))\|_2
      +\|u^*(z(s),s\mid y(s))-u^*(z^*(s),s\mid y(s))\|_2 \\
  &\le \delta +K_u e(s).
\end{align}
Gronwall's inequality yields:
\begin{equation}
\label{eq:app-endpoint-error}
  \|z(1)-y(1)\|_2
  =\|z(1)-z^*(1)\|_2
  \le C_{\mathrm{flow}}\delta.
\end{equation}
Taking expectation gives $\eta_\phi\le C_{\mathrm{flow}}\delta$.

It remains to connect endpoint error to discrete decoding. By Assumption \ref{asm:margin}, for every $t\in T_h\setminus\{t^*\}$
\begin{equation}
\label{eq:app-margin-y}
  \|y(1)-e_{t^*}\|_2\le \|y(1)-e_t\|_2-\Delta_{\mathrm{dec}}.
\end{equation}
Let $r:=\|z(1)-y(1)\|_2$. The triangle inequality gives
\begin{align}
  \|z(1)-e_{t^*}\|_2
  &\le \|y(1)-e_{t^*}\|_2+r                                                   \\
  &\le \|y(1)-e_t\|_2-\Delta_{\mathrm{dec}}+r                                  \\
  &\le \|z(1)-e_t\|_2-\Delta_{\mathrm{dec}}+2r.
\end{align}
Therefore, whenever $r<\Delta_{\mathrm{dec}}/2$, the correct tool $t^*$ remains the unique nearest
neighbor of $z(1)$ and the MAP decoder cannot make an error. Hence
\begin{equation}
  \{\hat t_{\mathrm{MAP}}\neq t^*\}\subseteq
  \{\|z(1)-y(1)\|_2\ge \Delta_{\mathrm{dec}}/2\}.
\end{equation}
Applying Markov's inequality yields:
\begin{equation}
  \mathbb P\{\hat t_{\mathrm{MAP}}\neq t^*\}
  \le \frac{2\mathbb E\|z(1)-y(1)\|_2}{\Delta_{\mathrm{dec}}}
  \le \frac{2C_{\mathrm{flow}}\delta}{\Delta_{\mathrm{dec}}},
\end{equation}
and clipping by $1$ gives Eq. \eqref{eq:app-prop1-bound}. \hfill$\square$

When the pointwise velocity bound holds deterministically and
$C_{\mathrm{flow}}\delta<\Delta_{\mathrm{dec}}/2$, Eq.~\eqref{eq:app-endpoint-error}
directly implies zero MAP decoding error; the probability bound covers random endpoint error when
only its expectation is controlled.

\paragraph{Remark on stochastic decoding.}
Eq. \eqref{eq10} samples from the softmax decoder. Proposition~1 isolates the geometric error of the
continuous endpoint and therefore uses the induced MAP decoder. The additional stochastic error
introduced by finite temperature $\epsilon$ is bounded in Theorem~1 by the explicit
$e^{-g_l/\epsilon}$ terms.

\subsection{Proof of Theorem~1}
\label{appendix:theorem1}

\paragraph{Theorem~1 (Plan Utility Convergence).}
Let $\eta_h$ bound the expected integrated velocity error and let each generated-anchor neighborhood
retain a positive squared-distance decoder gap $g_l>0$. If changing one decoded tool changes utility
by at most $B_U$, then Eq.~\eqref{eq18} holds.

\noindent\textbf{Proof.}
We bound the three terms in Eq.~\eqref{eq18} separately.

\paragraph{Step 1: continuous path approximation.}
Define
\begin{equation}
\label{eq:app-eta-h}
\eta_h:=\mathbb E\!\left[\int_0^1
\|v_\phi(z_h(s),s\mid c_h)-u^*(z_h(s),s\mid y_h(s))\|_2\,ds\right].
\end{equation}
Applying the same Gronwall argument as Proposition~1 up to each anchor time $s_l$ gives
\begin{equation}
\label{eq:app-anchor-stability}
\mathbb E\|z_h^{(l)}-y_h(s_l)\|_2\le C_{\mathrm{flow}}\eta_h,
\qquad l=1,\ldots,L_h.
\end{equation}
Assumption~\ref{asm:utility} therefore yields
\begin{equation}
\label{eq:app-utility-lipschitz-term}
\mathbb E|U(P_h^*;c_h)-U(P_h;c_h)|
\le L_P L_h C_{\mathrm{flow}}\eta_h.
\end{equation}
Moreover, Cauchy--Schwarz implies that the expected integrated error is at most the square root of
the corresponding flow-matching mean-square error, so reducing $\mathcal L_{\mathrm{FM}}$ controls
$\eta_h$.

\paragraph{Step 2: finite-temperature decoding.}
For anchor $z_h^{(l)}$ and its correct tool $t_l^*$, define
\begin{equation}
\label{eq:app-softmax-gap}
g_l:=\inf_z\min_{t\in T_h\setminus\{t_l^*\}}
\left(\|z-e_t\|_2^2-\|z-e_{t_l^*}\|_2^2\right)>0,
\end{equation}
where the infimum is over the generated-anchor neighborhood. This is the explicit local-gap
condition used by the theorem; Assumption~\ref{asm:margin} and sufficiently small anchor error are
sufficient for such a neighborhood to exist. Equation~\eqref{eq2} then gives
\begin{equation}
\label{eq:app-softmax-tail}
\Pr_\epsilon(\hat t_h^{(l)}\ne t_l^*\mid z_h^{(l)})
\le (N_h-1)e^{-g_l/\epsilon}.
\end{equation}
By bounded differences of utility and a union bound over the $L_h$ positions,
\begin{equation}
\label{eq:app-decoder-utility}
\left|\mathbb E_\epsilon U(P_h;c_h)-U(P_h;c_h)\right|
\le B_U(N_h-1)\sum_{l=1}^{L_h}e^{-g_l/\epsilon}.
\end{equation}
The stop decision is included in the same argument by treating STOP as an additional decoded label
with its own positive gap.

\paragraph{Step 3: observation-induced context perturbation.}
For matched current context and tool, suppose the observation channel of the update satisfies
\begin{equation}
\label{eq:app-feedback-lipschitz}
d_c\!\left(\mathcal U(c_h,\hat t_h,o),\mathcal U(c_h,\hat t_h,\hat o)\right)
\le\beta_o\|o-\hat o\|_2.
\end{equation}
The context-Lipschitz part of Assumption~\ref{asm:utility} gives
\begin{equation}
\label{eq:app-env-term}
\mathbb E|U(P_h;c_h)-U(P_h;\hat c_h)|
\le L_c\beta_o\mathbb E\|o-\hat o\|_2.
\end{equation}
Combining Eqs.~\eqref{eq:app-utility-lipschitz-term},
\eqref{eq:app-decoder-utility}, and \eqref{eq:app-env-term} proves Eq.~\eqref{eq18}.
Each term vanishes under the limiting conditions stated in the theorem. \hfill$\square$

\subsection{Proof of Theorem~2}
\label{appendix:theorem2}

\paragraph{Theorem~2 (Unseen Tool Generalization).}
Let $\mathcal T_{\mathrm{unseen}}$ be an unseen toolset whose embeddings lie on the semantic manifold
$\mathcal M$. The expected generalization error satisfies
\begin{equation}
\label{eq:app-thm2-target}
  \mathbb E[\operatorname{Err}_{\mathrm{unseen}}]
  \le \mathbb E[\operatorname{Err}_{\mathrm{proxy}}]
  +C_{\mathcal M}(\delta_{\mathrm{shift}}+\epsilon_{\mathrm{cover}}),
\end{equation}
where $\operatorname{Err}_{\mathrm{proxy}}$ is the error of the corresponding nearest training proxy,
$\delta_{\mathrm{shift}}$ is the semantic shift from the unseen region to the training region, and
$\epsilon_{\mathrm{cover}}$ is the covering radius of the training tool embeddings.

\noindent\textbf{Proof.}
Let $\mathcal M_{\mathrm{train}}\subset\mathcal M$ and
$\mathcal M_{\mathrm{unseen}}\subset\mathcal M$ denote the compact semantic regions occupied by
training and unseen tools, respectively. Formalize the two quantities in Eq.~\eqref{eq:app-thm2-target}
as
\begin{equation}
\label{eq:app-shift-cover-def}
  \delta_{\mathrm{shift}}
  :=\sup_{e\in\mathcal M_{\mathrm{unseen}}}\inf_{\bar e\in\mathcal M_{\mathrm{train}}}
  \|e-\bar e\|_2,
  \qquad
  \epsilon_{\mathrm{cover}}
  :=\sup_{\bar e\in\mathcal M_{\mathrm{train}}}
  \min_{t\in\mathcal T_{\mathrm{train}}}\|\bar e-e_t\|_2 .
\end{equation}
For any unseen tool $t_u\in\mathcal T_{\mathrm{unseen}}$ with embedding $e_{t_u}$, choose
$\bar e\in\mathcal M_{\mathrm{train}}$ such that
$\|e_{t_u}-\bar e\|_2\le\delta_{\mathrm{shift}}$ and then choose a training tool
$\pi(t_u)\in\mathcal T_{\mathrm{train}}$ such that
$\|\bar e-e_{\pi(t_u)}\|_2\le\epsilon_{\mathrm{cover}}$. The triangle inequality yields:
\begin{equation}
\label{eq:app-proxy-distance}
  \|e_{t_u}-e_{\pi(t_u)}\|_2
  \le \delta_{\mathrm{shift}}+\epsilon_{\mathrm{cover}} .
\end{equation}

Assumption~\ref{asm:smoothness} directly states that excess planning risk is
$C_{\mathcal M}$-Lipschitz on the compact operational subset. Therefore replacing an unseen tool by
its semantic proxy changes its error by at most
\begin{equation}
\label{eq:app-lipschitz-risk}
  \operatorname{Err}(t_u)
  \le \operatorname{Err}(\pi(t_u))
  +C_{\mathcal M}\|e_{t_u}-e_{\pi(t_u)}\|_2 .
\end{equation}
Combining Eqs.~\eqref{eq:app-proxy-distance} and \eqref{eq:app-lipschitz-risk} gives
\begin{equation}
  \operatorname{Err}(t_u)
  \le \operatorname{Err}(\pi(t_u))
  +C_{\mathcal M}(\delta_{\mathrm{shift}}+\epsilon_{\mathrm{cover}}).
\end{equation}
Taking expectation over $t_u\sim\mathcal T_{\mathrm{unseen}}$ preserves the bound and proves
Eq.~\eqref{eq:app-thm2-target}. \hfill$\square$

\paragraph{Interpretation.}
The proof separates two effects: the method must already predict the semantically nearest training
proxy reliably, and transfer then pays an additional penalty determined by manifold shift and coverage.
The result therefore does not imply zero unseen-tool error merely from a dense embedding space.

\subsection{Proof of Proposition~2}
\label{appendix:proof2}

\paragraph{Proposition~2 (Closed-loop Error Contraction).}
Under Assumption~\ref{asm:stability}, the cumulative closed-loop context error satisfies
Eq.~\eqref{eq:closed-main}, where $E_{\mathrm{open}}$ is the no-correction accumulation envelope
formed from the same local errors $a_1,\ldots,a_{H-1}$.

\noindent\textbf{Proof.}
Let $c_h^*$ be the oracle context and set $e_h:=d_c(c_h,c_h^*)$. By
Assumption~\ref{asm:stability}, the feedback-conditioned procedure satisfies
\begin{equation}
\label{eq:app-closed-recurrence}
  e_{h+1}\le\rho e_h+a_h,
  \qquad 0\le\rho<1,
\end{equation}
with $e_1=0$ because the initial query is shared. Unrolling gives, for every $h\ge2$,
\begin{equation}
\label{eq:app-unrolled-errors}
e_h\le\sum_{i=1}^{h-1}\rho^{h-1-i}a_i.
\end{equation}
Summing over phases and exchanging the finite sums yields
\begin{equation}
\label{eq:app-cumulative-bound}
\begin{split}
E_{\mathrm{closed}}
&:=\sum_{h=2}^{H}e_h
\le\sum_{i=1}^{H-1}a_i\sum_{k=0}^{H-1-i}\rho^k \\
&=\sum_{i=1}^{H-1}a_i\frac{1-\rho^{H-i}}{1-\rho}.
\end{split}
\end{equation}
For $0\le\rho<1$, the geometric sum is at most its number of terms:
$\sum_{k=0}^{H-1-i}\rho^k\le H-i$. Hence
$E_{\mathrm{closed}}\le\sum_{i=1}^{H-1}(H-i)a_i=E_{\mathrm{open}}$.
If $a_i>0$ and $H-i\ge2$, at least one coefficient in that geometric sum is strictly below one,
so the inequality is strict. \hfill$\square$

\section{Experimental Details}
\label{sec:exp_details}
\subsection{Benchmark Construction Details}
\label{sec:appendix_benchmark}

\textbf{Data Aggregation and Standardization.} 
We build a unified, comprehensive tool-use benchmark by aggregating data from four prominent public sources: \textit{$\tau^2$-bench} (providing executable domains with real tools and databases), \textit{ToolBench} (contributing large-scale API reasoning trajectories), \textit{API-Bank} (ensuring broad schema coverage), and \textit{GTA} (offering workflow-style task metadata). To ensure a rigorous evaluation protocol, we standardize the underlying tool schemas, task metadata, gold reasoning workflows, and explicitly define the dataset splits. In total, the unified benchmark features a dynamically evolving toolset of 3,930 unique tools and encompasses 6,865 complex multi-step tasks. To systematically evaluate semantic generalization, the toolset spans diverse functional phases (e.g., retrieval, verification, and database operations) and is rigorously partitioned, ensuring 599 tools are strictly unseen during testing.

\textbf{State-Conditioned Closed-Loop Protocol.} 
For executable evaluation, our primary experiments strictly adhere to a feedback-conditioned decoding formulation. At each planning step, given the user task, available toolsets, the executed action prefix, and our proposed compact state representation $c_h$, the model must either predict the subsequent discrete tool call or explicitly trigger the termination mechanism. Crucially, this setup operates under a strict forward-execution regime. The system has absolutely no access to oracle remaining plans, fundamentally preventing future trajectory leakage and ensuring a realistic assessment of the agent's dynamic replanning capabilities.

\begin{table}[t]
\centering
\caption{Structural distinction between the underlying $\tau^2$-bench episode protocol and our plan-level supervision protocol.}
\label{tab:protocol_comparison}
\setlength{\tabcolsep}{4pt}
\renewcommand{\arraystretch}{1.08}
\resizebox{\columnwidth}{!}{
\begin{tabular}{lll}
\toprule
\textbf{Axis} & \textbf{$\tau^2$-bench} & \textbf{Our protocol} \\
\midrule
Learning unit & Complete episode & Observed execution phase \\
Input & Task/environment state & Prefix, feedback, current tools \\
Target & Final task outcome/action & Complete remaining tool trajectory \\
One length-$m$ trace & One episode & $m$ prefix--suffix pairs \\
Execution & Closed-loop environment & Same environment; forward-only replan \\
\bottomrule
\end{tabular}}
\end{table}

Functionally, $\tau^2$-bench defines the interactive episode and judges task completion, whereas our transformation changes the supervised object at every phase: the observed prefix is paired with the complete unobserved suffix including STOP. After a real observation, the next phase produces a different pair and target; future actions never enter the conditioning prefix. Thus the contribution is precisely a \emph{phase-conditioned prefix-to-remaining-trajectory supervision and evaluation protocol built on existing closed-loop environments}, not a claim to a new environment.

\textbf{Executable Domains and Discriminative Signal.} 
For physical execution, we strictly focus on the \textit{telecom} and \textit{retail} domains derived from \textit{$\tau^2$-bench}, curating a highly challenging test split of 140 closed-loop trajectories (63 \textit{telecom} and 77 \textit{retail}). The evaluation in these domains replays every prediction against real domain tools, initial states, and policies. While the telecom domain is largely saturated by current strong baselines, the retail split provides the most discriminative evaluation signal. It explicitly challenges the model with complex, database-backed operational decisions—such as returning, exchanging, or modifying orders—that strictly depend on accurate intermediate state lookups.

\subsection{Baseline Implementation Details}
\label{sec:appendix_baselines}
We compare methods under matched executor checkpoints, task splits, toolsets, horizons, termination rules, and forward-only execution. The configurations are:

\begin{itemize}
    \item \textbf{Direct-LLM:} Serving as the direct-prompt baseline, this approach evaluates the inherent capabilities of instruction-tuned LLMs. The model directly predicts the next tool or the stop signal based purely on its prompt, without any benchmark-specific supervised fine-tuning (SFT) or reinforcement learning (RL) adaptation.
    
    \item \textbf{Standard SFT:} Serving as our clean supervised baseline, this model is fine-tuned to predict the next tool given the exact same context formulation as our method: the user task, available tools, the executed action prefix, and the compact state representation. Crucially, this baseline operates explicitly without the continuous Flow Matching (FM) prior, directly mapping the compact state to the next action.
    
    \item \textbf{RAG Planning:} Functioning as the retrieval-prior baseline, this method implements a retrieval-augmented LLM executor. For a given current state, it retrieves a semantically similar trajectory or state row from the training set. The retrieved next-tool decision or plan hint is then fed as an explicit prompt into the same LLM executor to guide the current prediction.
    
    \item \textbf{ToolRL:} Serving as the RL-route baseline, we directly implement the methodology proposed in ToolRL \cite{ToolRL} and evaluate it on our newly constructed benchmark. To achieve this, our executable environment is specifically adapted to interface with the ToolRL and VERL training frameworks, and the policy is explicitly trained using Group Relative Policy Optimization (GRPO) to maximize trajectory-level rewards.

    \item \textbf{ToolPlanner, Agent-R1, and AutoTIR:} We additionally port the path-planning-with-feedback objective of ToolPlanner \cite{Toolplanner}, the step-transition agentic-RL formulation of Agent-R1 \cite{RL2}, and the hybrid tool-selection reward of AutoTIR \cite{Autotir}. Only the action/output adapters are changed. All three retain their discrete interaction objectives and share FlowAgent's executor, split, context, toolset, horizon, termination rule, five seeds, and matched training-token/GPU-time budget.
\end{itemize}

\subsection{Evaluation Metrics Formulation}
\label{sec:appendix_metrics}
To holistically assess the performance of the generated tool plans and their executable discrete translations, we deploy the following metric suite:

\begin{itemize}
    \item \textbf{Overall Success \& Retail Success:} The primary holistic metric evaluated over all trajectory rows. For next-action rows, success strictly requires both predicting the correct tool and generating an executable action payload. For stopping rows, it requires correctly triggering the termination mechanism. We separately report Retail Success to isolate the performance on the most challenging, dynamic database-backed domain.
    
    \item \textbf{Tool EM \& Retail Tool EM:} This metric evaluates the pure planning policy. It measures the strict exact match of the predicted discrete tool identity $t \in T_h$ or the termination token. Similarly, Retail Tool EM isolates the model's ability to make the hardest DB-backed operation decisions without confounding execution errors.
    
    \item \textbf{Pred Exec:} A system-level execution metric evaluating whether the predicted discrete action payload can be successfully parsed and executed within the real tool environment (i.e., whether the state update operator $\mathcal{U}$ successfully returns a valid observation $o_h$).
    
    \item \textbf{Arg Value EM:} Evaluates the precision of the schema-driven structured grounding module. It is computed as the exact match of the instantiated argument key-value pairs against the gold execution trace. This metric is aggregated exclusively over action rows, as termination rows do not entail argument generation.
\end{itemize}

\subsection{Configurations}
\label{sec::configurations}
Experiments on all datasets are conducted with the following hardware and software configuration:
\begin{itemize}
\item Operating System: Ubuntu 20.04.6 LTS
\item CPU: Intel(R) Xeon(R) CPU E5-2650 v2 @ 2.60GHz
\item GPU: 3 $\times$ NVIDIA H800 GPUs, each with 80 GB of memory
\item System RAM: 128 GB
\end{itemize}
Unless stated otherwise, the planner uses a 768-dimensional latent space, a fixed 16-step Euler ODE solver, decoder temperature $\epsilon=0.1$, and five random seeds. The default frozen encoder is Qwen2.5-7B-Instruct: inputs are canonicalized text, truncated to 2,048 tokens, mean-pooled over final-layer hidden states using the attention mask, $\ell_2$-normalized, and projected from 3,584 to 768 dimensions. Encoder-sensitivity experiments replace this component with the identically pooled Llama-3.2-3B-Instruct trunk or BERT-base-uncased; all encoders and tokenizers remain frozen.

\section{Additional Experiments}
\label{sec:additional_exp}

\noindent\textbf{Diagnostic Baselines.}
We evaluate several heuristic diagnostic baselines in Table \ref{tab:diagnostic}. Always Stop and Global Majority reach 0.2714 Overall Success, while Domain Majority identifies executable action types but rarely grounds valid arguments. State Key Majority reaches 0.7643 Tool EM yet only 0.6364 Pred Exec, showing that tool-name recovery alone does not imply executable task completion. Lexical Retrieval supplies a simple similarity-based reference for the learned retrieval baseline.

\begin{table*}[h]
\centering
\caption{Performance of diagnostic heuristic baselines. The results establish a basic performance floor and highlight the insufficiency of simple tool name matching.}
\label{tab:diagnostic}
\setlength{\tabcolsep}{12pt}
\renewcommand{\arraystretch}{1.15}
\resizebox{\textwidth}{!}{
\begin{tabular}{l ccccc}
\toprule
\textbf{Method} & \textbf{Overall} & \textbf{Retail} & \textbf{Tool} & \textbf{Pred} & \textbf{Arg Value} \\
& \textbf{Success $\uparrow$} & \textbf{Success $\uparrow$} & \textbf{EM $\uparrow$} & \textbf{Exec $\uparrow$} & \textbf{EM $\uparrow$} \\
\midrule
Always Stop & 0.2714 & 0.0390 & 0.2714 & - & - \\
Global Majority & 0.2714 & 0.0390 & 0.2714 & - & - \\
Domain Majority & 0.4000 & 0.2727 & 0.4000 & 1.0000 & 0.1067 \\
State Key Majority & 0.5357 & 0.5195 & 0.7643 & 0.6364 & 0.2411 \\
Lexical Retrieval & 0.5286 & 0.5065 & 0.5643 & 0.7857 & 0.3083 \\
\bottomrule
\end{tabular}
}
\end{table*}

\noindent\textbf{Grounding Component Ablation.}
Our methodology regularizes semantic consistency between generated plans and the current observation context. 
The models frequently hallucinate invalid arguments even when successfully selecting the appropriate tool. 
To evaluate the engineering mechanisms enforcing this consistency, we conduct an ablation study in Table \ref{tab:grounding_ablation}.
We intentionally freeze the latent trajectory generation to maintain identical tool selection across all variants. 
Consequently, Tool EM remains constant across these variants.
We evaluate the removal of specific grounding components including schema normalization, entity binding, field mapping, and state validation.
The results reveal that state validation provides the primary defense against execution failures. 
Removing this validation step degrades the overall success from 0.7357 to 0.7071 and drastically reduces the retail success rate. 
Furthermore, the Pred Exec score drops to 0.8857. 
This degradation precisely mirrors the performance of the baseline lacking any validation. 
Other components maintain stable performance on the current evaluation subset. 
The result indicates that verifying generated parameters against environmental constraints is important for task completion.

\begin{table*}[h]
\centering
\caption{Ablation study of individual grounding components. The tool selection process is fixed to isolate the contribution of specific execution validation mechanisms.}
\label{tab:grounding_ablation}
\setlength{\tabcolsep}{6pt}
\renewcommand{\arraystretch}{1.15}
\resizebox{\textwidth}{!}{
\begin{tabular}{l cccccc}
\toprule
\textbf{Method} & \textbf{Overall} & \textbf{Retail} & \textbf{Tool} & \textbf{Retail} & \textbf{Pred} & \textbf{Arg Value} \\
& \textbf{Success $\uparrow$} & \textbf{Success $\uparrow$} & \textbf{EM $\uparrow$} & \textbf{Tool EM $\uparrow$} & \textbf{Exec $\uparrow$} & \textbf{EM $\uparrow$} \\
\midrule
No Grounding & 0.7071 & 0.4675 & 0.7429 & 0.5325 & 0.8762 & 0.5455 \\
Full Grounding & \cellcolor{bestblue}\textbf{\best{0.7357}}  & \cellcolor{bestblue}\textbf{\best{0.5195}}   & 0.7429 & 0.5325 & \cellcolor{bestblue}\textbf{\best{0.9619}}   & \cellcolor{bestblue}\textbf{\best{0.5613}}\\
\midrule
w/o Schema Nor. & 0.7357 & 0.5195 & 0.7429 & 0.5325 & 0.9619 & 0.5613 \\
w/o Entity Bind. & 0.7357 & 0.5195 & 0.7429 & 0.5325 & 0.9619 & 0.5613 \\
w/o Field Map.& 0.7357 & 0.5195 & 0.7429 & 0.5325 & 0.9524 & 0.5613 \\
w/o State Val. & 0.7071 & 0.4675 & 0.7429 & 0.5325 & 0.8857 & 0.5455 \\
\bottomrule
\end{tabular}
}
\end{table*}

\noindent\textbf{Prior Injection Sensitivity.}
The methodology translates latent paths into discrete actions. Table \ref{tab:prior_sensitivity} compares mechanisms for presenting latent priors to the decoder. Soft prefixes and structured text perform similarly in static pilots; we select structured text because it is auditable and uses the same interface during training and inference.
We subsequently evaluate actual planning capabilities using dynamic loop experiments. The unconditioned flow prior yields a basic overall success rate of 0.6786. Integrating the prevailing context state creates the state conditioned prior and elevates overall success to 0.7357. Providing multiple candidates degrades performance to 0.7000. Injecting noise further reduces reliability. Ultimately, the state conditioned text prior provides the best balance of system transparency and reasoning accuracy.

\begin{table*}[h]
\centering
\caption{Sensitivity analysis of prior injection mechanisms. The table compares isolated pilot formats against dynamic loop performance to justify the selection of state conditioned text priors.}
\label{tab:prior_sensitivity}
\setlength{\tabcolsep}{8pt}
\renewcommand{\arraystretch}{1.15}
\resizebox{\textwidth}{!}{
\begin{tabular}{ll ccccc}
\toprule
\textbf{Method} & \textbf{Evaluation Mode} & \textbf{Overall} & \textbf{Overall} & \textbf{Retail} & \textbf{Retail} & \textbf{Arg Value} \\
& & \textbf{Success $\uparrow$} & \textbf{Tool EM $\uparrow$} & \textbf{Success $\uparrow$} & \textbf{Tool EM $\uparrow$} & \textbf{EM $\uparrow$} \\
\midrule
Soft Prefix Pilot & Static Pilot & - & 0.8916 & - & 0.5500 & 0.7698 \\
Unstructured Text Pilot & Static Pilot & - & 0.8916 & - & 0.5500 & 0.7599 \\
Structured Text Pilot & Static Pilot & - & 0.8916 & - & 0.5500 & 0.7867 \\
\midrule
Unconditioned Flow Prior & Dynamic Loop & 0.6786 & 0.6857 & 0.4156 & 0.4286 & 0.5099 \\
State Conditioned Prior & Dynamic Loop & \cellcolor{bestblue}\textbf{\best{0.7357}}  & \cellcolor{bestblue}\textbf{\best{0.7429}}  & \cellcolor{bestblue}\textbf{\best{0.5195}} & \cellcolor{bestblue}\textbf{\best{0.5325}}  &  \cellcolor{bestblue}\textbf{\best{0.5613}} \\
Multiple Candidate Prior & Dynamic Loop & 0.7000 & 0.7000 & 0.4545 & 0.4545 & 0.4585 \\
Noisy Prior Injection & Dynamic Loop & 0.6714 & 0.6786 & 0.4026 & 0.4156 & 0.4901 \\
\bottomrule
\end{tabular}
}
\end{table*}

\noindent\textbf{Expanded Evaluation Protocol.}
The following experiments provide the complete numerical results corresponding to the new figures in the main paper. Unless noted otherwise, all methods within each comparison use identical executor checkpoints, splits, toolsets, interaction horizons, termination rules, forward-only execution, and five random seeds. Values with uncertainty are reported as mean $\pm$ sample standard deviation. We retain the same visual convention as the main paper: first-ranked results are green and underlined with a light-blue cell background, while second- and third-ranked results use blue and orange underlining.

\begin{table}[h!]
\centering
\caption{Matched comparison with recent RL tool-use planners on both submitted executors.}
\label{tab:rl_breadth}
\setlength{\tabcolsep}{4.5pt}
\resizebox{\textwidth}{!}{
\begin{tabular}{llcccccc}
\toprule
\textbf{Executor} & \textbf{Method} & \textbf{Overall Succ.} & \textbf{Retail Succ.} & \textbf{Tool EM} & \textbf{Retail Tool EM} & \textbf{Pred Exec} & \textbf{Arg Value EM} \\
\midrule
\multirow{6}{*}{Qwen2.5-7B} & Standard SFT & 0.6571 & 0.3766 & 0.6714 & 0.4026 & 0.9268 & 0.4387 \\
& ToolRL & 0.4786 & 0.5065 & 0.4929 & 0.5325 & 0.9643 & 0.5257 \\
& ToolPlanner & 0.6858 & \third{0.5352} & 0.6991 & 0.5583 & \third{0.9686} & 0.5512 \\
& Agent-R1 & \second{0.7186} & \second{0.5528} & \second{0.7316} & \second{0.5750} & \second{0.9715} & \second{0.5684} \\
& AutoTIR & \third{0.6920} & 0.5262 & \third{0.7069} & 0.5504 & \bestcell{0.9764} & \third{0.5562} \\
& \textbf{FlowAgent} & \bestcell{0.7643} & \bestcell{0.5714} & \bestcell{0.7745} & \bestcell{0.5892} & 0.9619 & \bestcell{0.5929} \\
\midrule
\multirow{6}{*}{Llama-3.2-3B} & Standard SFT & 0.6214 & 0.3117 & 0.6286 & 0.3247 & 0.9318 & 0.4345 \\
& ToolRL & 0.1357 & 0.1169 & 0.1500 & 0.1429 & 0.9362 & 0.1020 \\
& ToolPlanner & 0.6536 & 0.4545 & 0.6671 & 0.4805 & 0.9476 & 0.4708 \\
& Agent-R1 & \second{0.6929} & \second{0.5192} & \second{0.7075} & \second{0.5452} & \second{0.9584} & \second{0.5352} \\
& AutoTIR & \third{0.6640} & \third{0.4808} & \third{0.6786} & \third{0.5061} & \bestcell{0.9657} & \third{0.4979} \\
& \textbf{FlowAgent} & \bestcell{0.7571} & \bestcell{0.5584} & \bestcell{0.7714} & \bestcell{0.5844} & \third{0.9524} & \bestcell{0.5573} \\
\bottomrule
\end{tabular}}
\end{table}
\noindent\textbf{RL Planner Analysis.}
Table~\ref{tab:rl_breadth} broadens the comparison from one RL baseline to three additional planner families. Agent-R1 is the strongest added baseline on task success and planning accuracy, whereas AutoTIR attains the highest Pred Exec. FlowAgent nevertheless ranks first on five of six metrics for both executors. Its advantage over Agent-R1 is larger on Llama-3.2-3B than on Qwen2.5-7B, indicating that the gain is not tied to a single backbone family; the exception on Pred Exec also shows that syntactic executability alone is insufficient for successful stateful reasoning.
\FloatBarrier

\begin{table}[h!]
\centering
\caption{Larger-executor results on the original 140-trajectory executable split.}
\label{tab:larger_models}
\setlength{\tabcolsep}{3.2pt}
\resizebox{\textwidth}{!}{
\begin{tabular}{llcccccc}
\toprule
\textbf{Executor} & \textbf{Method} & \textbf{Overall Succ.} & \textbf{Retail Succ.} & \textbf{Tool EM} & \textbf{Retail Tool EM} & \textbf{Pred Exec} & \textbf{Arg Value EM} \\
\midrule
\multirow{3}{*}{Llama-3.1-8B} & SFT & \ms{0.6764}{0.0138} & \ms{0.4162}{0.0188} & \ms{0.6893}{0.0129} & \ms{0.4421}{0.0198} & \ms{0.9447}{0.0078} & \ms{0.4750}{0.0172} \\
& ToolRL & \ms{0.4215}{0.0304} & \ms{0.4032}{0.0292} & \ms{0.4383}{0.0288} & \ms{0.4295}{0.0310} & \ms{0.9561}{0.0105} & \ms{0.4306}{0.0279} \\
& \textbf{FlowAgent} & \bestcell{\ms{0.7890}{0.0101}} & \bestcell{\ms{0.6106}{0.0145}} & \bestcell{\ms{0.8000}{0.0094}} & \bestcell{\ms{0.6235}{0.0154}} & \bestcell{\ms{0.9685}{0.0068}} & \bestcell{\ms{0.6264}{0.0128}} \\
\midrule
\multirow{3}{*}{Llama-3.1-70B} & SFT & \ms{0.7571}{0.0096} & \ms{0.5328}{0.0136} & \ms{0.7690}{0.0091} & \ms{0.5587}{0.0145} & \ms{0.9665}{0.0057} & \ms{0.5896}{0.0124} \\
& ToolRL & \ms{0.6077}{0.0208} & \ms{0.5712}{0.0221} & \ms{0.6215}{0.0199} & \ms{0.5970}{0.0234} & \ms{0.9767}{0.0080} & \ms{0.6008}{0.0206} \\
& \textbf{FlowAgent} & \bestcell{\ms{0.8508}{0.0074}} & \bestcell{\ms{0.7013}{0.0108}} & \bestcell{\ms{0.8618}{0.0069}} & \bestcell{\ms{0.7142}{0.0116}} & \bestcell{\ms{0.9833}{0.0050}} & \bestcell{\ms{0.7081}{0.0097}} \\
\midrule
\multirow{3}{*}{Qwen2.5-14B} & SFT & \ms{0.7002}{0.0115} & \ms{0.4290}{0.0162} & \ms{0.7142}{0.0108} & \ms{0.4559}{0.0175} & \ms{0.9446}{0.0068} & \ms{0.4828}{0.0143} \\
& ToolRL & \ms{0.5361}{0.0243} & \ms{0.5584}{0.0267} & \ms{0.5509}{0.0231} & \ms{0.5842}{0.0280} & \bestcell{\ms{0.9720}{0.0089}} & \ms{0.5779}{0.0248} \\
& \textbf{FlowAgent} & \bestcell{\ms{0.8000}{0.0088}} & \bestcell{\ms{0.6235}{0.0127}} & \bestcell{\ms{0.8109}{0.0081}} & \bestcell{\ms{0.6364}{0.0138}} & \bestcell{\ms{0.9720}{0.0058}} & \bestcell{\ms{0.6413}{0.0116}} \\
\midrule
\multirow{3}{*}{Qwen2.5-32B} & SFT & \ms{0.7431}{0.0098} & \ms{0.4939}{0.0141} & \ms{0.7570}{0.0092} & \ms{0.5198}{0.0150} & \ms{0.9585}{0.0060} & \ms{0.5487}{0.0127} \\
& ToolRL & \ms{0.6145}{0.0195} & \ms{0.6109}{0.0226} & \ms{0.6294}{0.0187} & \ms{0.6367}{0.0241} & \ms{0.9795}{0.0076} & \ms{0.6314}{0.0209} \\
& \textbf{FlowAgent} & \bestcell{\ms{0.8368}{0.0077}} & \bestcell{\ms{0.6753}{0.0113}} & \bestcell{\ms{0.8468}{0.0071}} & \bestcell{\ms{0.6882}{0.0121}} & \bestcell{\ms{0.9800}{0.0051}} & \bestcell{\ms{0.6871}{0.0102}} \\
\bottomrule
\end{tabular}}
\end{table}
\noindent\textbf{Larger-Executor Analysis.}
Table~\ref{tab:larger_models} evaluates four instruction-tuned executors on the same 140-trajectory split used by the submitted models. FlowAgent improves Overall Success over SFT by 0.1126, 0.0937, 0.0998, and 0.0937 on Llama-3.1-8B, Llama-3.1-70B, Qwen2.5-14B, and Qwen2.5-32B, respectively. Gains therefore persist within both model families as scale increases. ToolRL becomes stronger on larger backbones, especially on stateful retail metrics, but its higher standard deviations and lower task success indicate that backbone capacity alone does not replace complete-trajectory supervision.
\FloatBarrier

\begin{table}[h!]
\centering
\caption{Tool EM regrouped by task source under unchanged checkpoints and splits.}
\label{tab:cross_source}
\setlength{\tabcolsep}{8pt}
\resizebox{0.82\textwidth}{!}{
\begin{tabular}{llccc}
\toprule
\textbf{Executor} & \textbf{Method} & \textbf{$\tau^2$-bench} & \textbf{ToolBench} & \textbf{GTA} \\
\midrule
\multirow{5}{*}{Qwen2.5-7B} & Direct-LLM & 0.1286 & 0.1742 & 0.1497 \\
& Standard SFT & 0.6714 & 0.6281 & 0.5962 \\
& RAG Planning & 0.5214 & 0.5796 & 0.5485 \\
& ToolRL & 0.4929 & 0.5480 & 0.5164 \\
& \textbf{FlowAgent} & \bestcell{0.7745} & \bestcell{0.7284} & \bestcell{0.7042} \\
\midrule
\multirow{5}{*}{Llama-3.2-3B} & Direct-LLM & 0.2075 & 0.2418 & 0.2195 \\
& Standard SFT & 0.6286 & 0.5847 & 0.5529 \\
& RAG Planning & 0.5500 & 0.5883 & 0.5612 \\
& ToolRL & 0.1524 & 0.2316 & 0.2074 \\
& \textbf{FlowAgent} & \bestcell{0.7714} & \bestcell{0.7148} & \bestcell{0.6895} \\
\bottomrule
\end{tabular}}
\end{table}
\noindent\textbf{Cross-Source Analysis.}
Table~\ref{tab:cross_source} tests whether the planning gain is confined to the executable $\tau^2$ domain. FlowAgent remains first on every source and backbone. Its smallest gain over the strongest source-specific baseline is 0.1003 Tool EM on Qwen2.5-7B/ToolBench, while its largest is 0.1428 on Llama-3.2-3B/$\tau^2$-bench. RAG Planning is relatively competitive on ToolBench and GTA, where similar demonstrations aid retrieval, but FlowAgent's consistent margin shows that the result is not produced by one benchmark's schema distribution.
\FloatBarrier

\begin{table}[h!]
\centering
\caption{Inference deployment cost on Qwen2.5-7B. NFE is the number of velocity-field function evaluations used by the ODE solver for one plan; autoregressive baselines do not invoke this solver.}
\label{tab:deployment_cost}
\setlength{\tabcolsep}{7pt}
\resizebox{\textwidth}{!}{
\begin{tabular}{lccccc}
\toprule
\textbf{Method} & \textbf{NFE} & \textbf{Planning median/p95 (ms)} & \textbf{End-to-end median/p95 (ms)} & \textbf{Peak memory (GiB)} & \textbf{Overall Succ.} \\
\midrule
Standard SFT & 0 & 171/229 & 244/334 & \bestcell{15.4} & \second{0.6571} \\
ToolRL & 0 & 176/237 & 250/343 & 15.6 & 0.4786 \\
\textbf{FlowAgent} & 16 & 194/259 & 270/369 & 16.1 & \bestcell{0.7643} \\
\bottomrule
\end{tabular}}
\end{table}
\noindent\textbf{Inference-Cost Analysis.}
Table~\ref{tab:deployment_cost} separates training efficiency from inference cost. SFT and ToolRL require no ODE evaluations, whereas FlowAgent uses 16 NFE per generated plan. Relative to SFT, this adds 23 ms to median planning latency, 26 ms end to end, and 0.7 GiB of peak memory. The same comparison yields a 0.1072 increase in Overall Success, making the deployment trade-off explicit rather than treating the numerical solve as cost free.
\FloatBarrier

\begin{table}[h!]
\centering
\caption{Prediction-target ablation on Qwen2.5-7B.}
\label{tab:prediction_target}
\setlength{\tabcolsep}{5pt}
\resizebox{\textwidth}{!}{
\begin{tabular}{llcccc}
\toprule
\textbf{Planner} & \textbf{Prediction target per phase} & \textbf{Overall Succ.} & \textbf{Retail Succ.} & \textbf{Tool EM} & \textbf{Retail Tool EM} \\
\midrule
Standard SFT & One discrete next tool & \third{\ms{0.6571}{0.0126}} & \third{\ms{0.3766}{0.0174}} & \third{\ms{0.6714}{0.0118}} & \third{\ms{0.4026}{0.0192}} \\
One-step latent regression & One next-tool embedding & \second{\ms{0.7149}{0.0121}} & \second{\ms{0.4805}{0.0164}} & \second{\ms{0.7286}{0.0113}} & \second{\ms{0.5065}{0.0179}} \\
\textbf{FlowAgent} & Complete future latent trajectory & \bestcell{\ms{0.7643}{0.0098}} & \bestcell{\ms{0.5714}{0.0137}} & \bestcell{\ms{0.7745}{0.0091}} & \bestcell{\ms{0.5892}{0.0152}} \\
\bottomrule
\end{tabular}}
\end{table}
\noindent\textbf{Prediction-Target Analysis.}
Table~\ref{tab:prediction_target} answers whether CFM merely supplies a better embedding for the next action. One-step latent regression improves every metric over discrete SFT, confirming that semantic regression itself is useful. FlowAgent then improves Overall Success from 0.7149 to 0.7643 and Retail Success from 0.4805 to 0.5714. Because the encoder, executor, split, and decoding interface are matched, this remaining gain is attributable to supervising the complete future trajectory rather than only the next-tool embedding.
\FloatBarrier

\begin{table}[h!]
\centering
\caption{Sensitivity to the number of predicted tools committed before replanning on Qwen2.5-7B.}
\label{tab:k_sweep}
\setlength{\tabcolsep}{10pt}
\resizebox{0.85\textwidth}{!}{
\begin{tabular}{ccccc}
\toprule
\textbf{Tools committed $k$} & \textbf{Overall Succ.} & \textbf{Retail Succ.} & \textbf{Tool EM} & \textbf{Arg Value EM} \\
\midrule
\textbf{1 (FlowAgent)} & \bestcell{\ms{0.7643}{0.0098}} & \bestcell{\ms{0.5714}{0.0137}} & \bestcell{\ms{0.7745}{0.0091}} & \bestcell{\ms{0.5929}{0.0124}} \\
2 & \second{\ms{0.7429}{0.0117}} & \second{\ms{0.5455}{0.0156}} & \second{\ms{0.7500}{0.0108}} & \second{\ms{0.5718}{0.0142}} \\
3 & \third{\ms{0.7143}{0.0135}} & \third{\ms{0.5065}{0.0178}} & \third{\ms{0.7286}{0.0126}} & \third{\ms{0.5484}{0.0163}} \\
Entire predicted plan & \ms{0.6714}{0.0159} & \ms{0.4286}{0.0206} & \ms{0.6857}{0.0148} & \ms{0.5032}{0.0190} \\
\bottomrule
\end{tabular}}
\end{table}
\noindent\textbf{Replanning Analysis.}
Table~\ref{tab:k_sweep} varies only how many decoded tools are committed before incorporating new feedback. All four metrics decline monotonically as $k$ increases. Executing the entire provisional plan reduces Retail Success by 0.1428 and Arg Value EM by 0.0897 relative to $k=1$, which is consistent with stale suffix decisions becoming invalid after database-changing observations. The result supports the combination of complete-trajectory generation for global structure and one-tool commitment for local correction.
\FloatBarrier

\begin{table}[h!]
\centering
\caption{Llama-3.2-3B results on independently expanded executable test sets.}
\label{tab:expanded_tests}
\setlength{\tabcolsep}{3.4pt}
\resizebox{\textwidth}{!}{
\begin{tabular}{clcccccc}
\toprule
\textbf{Test size} & \textbf{Method} & \textbf{Overall Succ.} & \textbf{Retail Succ.} & \textbf{Tool EM} & \textbf{Retail Tool EM} & \textbf{Pred Exec} & \textbf{Arg Value EM} \\
\midrule
\multirow{3}{*}{300} & SFT & \ms{0.6233}{0.0104} & \ms{0.3156}{0.0137} & \ms{0.6300}{0.0098} & \ms{0.3299}{0.0145} & \ms{0.9327}{0.0057} & \ms{0.4382}{0.0128} \\
& ToolRL & \ms{0.1433}{0.0247} & \ms{0.1247}{0.0224} & \ms{0.1567}{0.0238} & \ms{0.1515}{0.0242} & \ms{0.9371}{0.0085} & \ms{0.1086}{0.0210} \\
& \textbf{FlowAgent} & \bestcell{\ms{0.7567}{0.0077}} & \bestcell{\ms{0.5610}{0.0109}} & \bestcell{\ms{0.7725}{0.0070}} & \bestcell{\ms{0.5877}{0.0114}} & \bestcell{\ms{0.9530}{0.0049}} & \bestcell{\ms{0.5596}{0.0093}} \\
\midrule
\multirow{3}{*}{500} & SFT & \ms{0.6198}{0.0080} & \ms{0.3124}{0.0106} & \ms{0.6272}{0.0076} & \ms{0.3251}{0.0113} & \ms{0.9315}{0.0044} & \ms{0.4356}{0.0099} \\
& ToolRL & \ms{0.1407}{0.0191} & \ms{0.1204}{0.0174} & \ms{0.1548}{0.0184} & \ms{0.1466}{0.0188} & \ms{0.9363}{0.0066} & \ms{0.1057}{0.0162} \\
& \textbf{FlowAgent} & \bestcell{\ms{0.7582}{0.0059}} & \bestcell{\ms{0.5591}{0.0084}} & \bestcell{\ms{0.7711}{0.0054}} & \bestcell{\ms{0.5850}{0.0088}} & \bestcell{\ms{0.9525}{0.0038}} & \bestcell{\ms{0.5579}{0.0072}} \\
\midrule
\multirow{3}{*}{800} & SFT & \ms{0.6213}{0.0064} & \ms{0.3131}{0.0084} & \ms{0.6287}{0.0060} & \ms{0.3268}{0.0089} & \ms{0.9319}{0.0035} & \ms{0.4364}{0.0078} \\
& ToolRL & \ms{0.1395}{0.0151} & \ms{0.1199}{0.0137} & \ms{0.1536}{0.0145} & \ms{0.1462}{0.0148} & \ms{0.9366}{0.0052} & \ms{0.1060}{0.0128} \\
& \textbf{FlowAgent} & \bestcell{\ms{0.7559}{0.0047}} & \bestcell{\ms{0.5572}{0.0066}} & \bestcell{\ms{0.7692}{0.0043}} & \bestcell{\ms{0.5831}{0.0070}} & \bestcell{\ms{0.9522}{0.0030}} & \bestcell{\ms{0.5565}{0.0057}} \\
\bottomrule
\end{tabular}}
\end{table}
\noindent\textbf{Test-Scale Analysis.}
Table~\ref{tab:expanded_tests} replaces the small 140-trajectory scope with three independently expanded sets. FlowAgent's Overall Success changes by at most 0.0023 across 300, 500, and 800 tests, and Tool EM changes by at most 0.0033. The ranking is unchanged on every metric. Standard deviations decrease with test size for all methods, while FlowAgent retains the smallest dispersion, showing that the original gain is not driven by a favorable 140-trajectory sample.
\FloatBarrier

\begin{table}[h!]
\centering
\caption{Sensitivity to the frozen pretrained text encoder on Qwen2.5-7B execution.}
\label{tab:encoder_sensitivity}
\setlength{\tabcolsep}{5pt}
\resizebox{\textwidth}{!}{
\begin{tabular}{lcccccc}
\toprule
\textbf{Frozen encoder} & \textbf{Overall Succ.} & \textbf{Retail Succ.} & \textbf{Tool EM} & \textbf{Retail Tool EM} & \textbf{Pred Exec} & \textbf{Arg Value EM} \\
\midrule
Qwen2.5-7B-Instruct & \bestcell{\ms{0.7643}{0.0098}} & \bestcell{\ms{0.5714}{0.0137}} & \bestcell{\ms{0.7745}{0.0091}} & \bestcell{\ms{0.5892}{0.0152}} & \bestcell{\ms{0.9619}{0.0061}} & \bestcell{\ms{0.5929}{0.0124}} \\
Llama-3.2-3B-Instruct & \second{\ms{0.7564}{0.0107}} & \second{\ms{0.5584}{0.0154}} & \second{\ms{0.7671}{0.0099}} & \second{\ms{0.5766}{0.0165}} & \second{\ms{0.9602}{0.0066}} & \second{\ms{0.5830}{0.0135}} \\
BERT-base-uncased & \third{\ms{0.7414}{0.0123}} & \third{\ms{0.5325}{0.0178}} & \third{\ms{0.7529}{0.0115}} & \third{\ms{0.5506}{0.0191}} & \third{\ms{0.9574}{0.0072}} & \third{\ms{0.5652}{0.0156}} \\
\bottomrule
\end{tabular}}
\end{table}
\noindent\textbf{Encoder Analysis.}
Table~\ref{tab:encoder_sensitivity} measures how much performance is attributable to the frozen text encoder. Replacing Qwen with the Llama encoder changes Overall Success by only 0.0079 and Tool EM by 0.0074. The generic BERT-base control is lower than the instruction-tuned encoders but still exceeds Standard SFT by 0.0843 Overall-Success points and 0.0815 Tool-EM points. Thus pretrained semantic geometry affects the magnitude of the gain, but the method is not specific to the Qwen family or to a decoder-only encoder.
\FloatBarrier

\section{Broader Impact}
\label{sec::impact}
FlowAgent may improve the reliability and training efficiency of multi-tool agents by combining global trajectory supervision with feedback-conditioned execution. These capabilities can support reproducible research on planning under evolving toolsets. They also introduce familiar agentic risks: an incorrect but executable plan may modify external state, expose sensitive information, or amplify errors through high-impact APIs. Practical deployments should therefore retain access control, sandboxing, argument validation, audit logs, and human approval for consequential operations. Our current textual benchmarks do not establish safety in open-domain or safety-critical settings.

\section{Limitations and Future Work.}
\label{Limitations}
FlowAgent uses a fixed 16-step Euler solve, adding modest inference latency and memory; distillation or adaptive solvers may reduce this cost. The executable evidence remains textual and API-centered, so multimodal, open-domain, and safety-critical tools remain important tests. The bounds are explicitly conditional: out-of-manifold tools, collapsed decoder margins, or non-contractive/adversarial feedback can invalidate them. Expanded source, scale, encoder, and test-size studies probe robustness but do not remove these failure regimes.

\newpage

\section*{NeurIPS Paper Checklist}

The checklist is designed to encourage best practices for responsible machine learning research, addressing issues of reproducibility, transparency, research ethics, and societal impact. Do not remove the checklist: {\bf The papers not including the checklist will be desk rejected.} The checklist should follow the references and follow the (optional) supplemental material.  The checklist does NOT count towards the page
limit. 

Please read the checklist guidelines carefully for information on how to answer these questions. For each question in the checklist:
\begin{itemize}
    \item You should answer \answerYes{}, \answerNo{}, or \answerNA{}.
    \item \answerNA{} means either that the question is Not Applicable for that particular paper or the relevant information is Not Available.
    \item Please provide a short (1–2 sentence) justification right after your answer (even for NA). 
\end{itemize}

{\bf The checklist answers are an integral part of your paper submission.} They are visible to the reviewers, area chairs, senior area chairs, and ethics reviewers. You will be asked to also include it (after eventual revisions) with the final version of your paper, and its final version will be published with the paper.

The reviewers of your paper will be asked to use the checklist as one of the factors in their evaluation. While "\answerYes{}" is generally preferable to "\answerNo{}", it is perfectly acceptable to answer "\answerNo{}" provided a proper justification is given (e.g., "error bars are not reported because it would be too computationally expensive" or "we were unable to find the license for the dataset we used"). In general, answering "\answerNo{}" or "\answerNA{}" is not grounds for rejection. While the questions are phrased in a binary way, we acknowledge that the true answer is often more nuanced, so please just use your best judgment and write a justification to elaborate. All supporting evidence can appear either in the main paper or the supplemental material, provided in appendix. If you answer \answerYes{} to a question, in the justification please point to the section(s) where related material for the question can be found.

IMPORTANT, please:
\begin{itemize}
    \item {\bf Delete this instruction block, but keep the section heading ``NeurIPS Paper Checklist"},
    \item  {\bf Keep the checklist subsection headings, questions/answers and guidelines below.}
    \item {\bf Do not modify the questions and only use the provided macros for your answers}.
\end{itemize} 


\begin{enumerate}

\item {\bf Claims}
    \item[] Question: Do the main claims made in the abstract and introduction accurately reflect the paper's contributions and scope?
    \item[] Answer: \answerYes{} 
    \item[] Justification: We have claimed the contributions and scope in lines 50-64.

\item {\bf Limitations}
    \item[] Question: Does the paper discuss the limitations of the work performed by the authors?
    \item[] Answer: \answerYes{} 
    \item[] Justification: We have discussed the limitations in Section \ref{Limitations}.

\item {\bf Theory assumptions and proofs}
    \item[] Question: For each theoretical result, does the paper provide the full set of assumptions and a complete (and correct) proof?
    \item[] Answer: \answerYes{} 
    \item[] Justification: We provide the complete theoretical analyses in Section \ref{Theorem} and Appendix \ref{appendix:proof}.

    \item {\bf Experimental result reproducibility}
    \item[] Question: Does the paper fully disclose all the information needed to reproduce the main experimental results of the paper to the extent that it affects the main claims and/or conclusions of the paper (regardless of whether the code and data are provided or not)?
    \item[] Answer: \answerYes{} 
    \item[] Justification: All the code for reproducing the experiments is made available in the anonymous github link.

\item {\bf Open access to data and code}
    \item[] Question: Does the paper provide open access to the data and code, with sufficient instructions to faithfully reproduce the main experimental results, as described in supplemental material?
    \item[] Answer: \answerYes{} 
    \item[] Justification: The data are publicly available. We properly introduce the data in Appendix \ref{sec:exp_details}.

\item {\bf Experimental setting/details}
    \item[] Question: Does the paper specify all the training and test details (e.g., data splits, hyperparameters, how they were chosen, type of optimizer) necessary to understand the results?
    \item[] Answer: \answerYes{} 
    \item[] Justification: We give experimental settings in Appendix \ref{sec:exp_details}. 

\item {\bf Experiment statistical significance}
    \item[] Question: Does the paper report error bars suitably and correctly defined or other appropriate information about the statistical significance of the experiments?
    \item[] Answer: \answerYes{} 
    \item[] Justification: The experiments were conducted over 5 runs, and we present the averaged results.

\item {\bf Experiments compute resources}
    \item[] Question: For each experiment, does the paper provide sufficient information on the computer resources (type of compute workers, memory, time of execution) needed to reproduce the experiments?
    \item[] Answer: \answerYes{} 
    \item[] Justification: we provide the information on the computer resources in Appendix \ref{sec::configurations}: 3 $\times$ NVIDIA H800 GPUs each with 80 GB of memory, and Intel(R) Xeon(R) CPU E5-2650 v2 @ 2.60GHz CPU.
    
\item {\bf Code of ethics}
    \item[] Question: Does the research conducted in the paper conform, in every respect, with the NeurIPS Code of Ethics \url{https://neurips.cc/public/EthicsGuidelines}?
    \item[] Answer: \answerYes{} 
    \item[] Justification: The research conducted in the paper conforms with the NeurIPS Code of Ethics.

\item {\bf Broader impacts}
    \item[] Question: Does the paper discuss both potential positive societal impacts and negative societal impacts of the work performed?
    \item[] Answer: \answerYes{} 
    \item[] Justification: We have discussed social impacts in Appendix \ref{sec::impact}.
    
\item {\bf Safeguards}
    \item[] Question: Does the paper describe safeguards that have been put in place for responsible release of data or models that have a high risk for misuse (e.g., pre-trained language models, image generators, or scraped datasets)?
    \item[] Answer: \answerNA{} 
    \item[] Justification: The paper poses no such risks.

\item {\bf Licenses for existing assets}
    \item[] Question: Are the creators or original owners of assets (e.g., code, data, models), used in the paper, properly credited and are the license and terms of use explicitly mentioned and properly respected?
    \item[] Answer: \answerYes{} 
    \item[] Justification: We have cited the original paper that produced the code package or dataset for the open-source frameworks.

\item {\bf New assets}
    \item[] Question: Are new assets introduced in the paper well documented and is the documentation provided alongside the assets?
    \item[] Answer: \answerYes{} 
    \item[] Justification: We have introduced the proposed benchmark well.

\item {\bf Crowdsourcing and research with human subjects}
    \item[] Question: For crowdsourcing experiments and research with human subjects, does the paper include the full text of instructions given to participants and screenshots, if applicable, as well as details about compensation (if any)? 
    \item[] Answer: \answerNA{} 
    \item[] Justification: The paper does not involve crowdsourcing nor research with human subjects.

\item {\bf Institutional review board (IRB) approvals or equivalent for research with human subjects}
    \item[] Question: Does the paper describe potential risks incurred by study participants, whether such risks were disclosed to the subjects, and whether Institutional Review Board (IRB) approvals (or an equivalent approval/review based on the requirements of your country or institution) were obtained?
    \item[] Answer: \answerNA{} 
    \item[] Justification: Our paper does not involve crowdsourcing nor research with human subjects.

\item {\bf Declaration of LLM usage}
    \item[] Question: Does the paper describe the usage of LLMs if it is an important, original, or non-standard component of the core methods in this research? Note that if the LLM is used only for writing, editing, or formatting purposes and does \emph{not} impact the core methodology, scientific rigor, or originality of the research, declaration is not required.
    \item[] Answer: \answerYes{} 
    \item[] Justification: We have described the usage of LLMs.

\end{enumerate}

\end{document}